\documentclass{article}

\usepackage[preprint,nonatbib]{neurips_2026}

\usepackage[utf8]{inputenc} 
\usepackage[T1]{fontenc}    
\usepackage{hyperref}       
\usepackage{url}            
\usepackage{booktabs}       
\usepackage{amsfonts}       
\usepackage{nicefrac}       
\usepackage{microtype}      
\usepackage{xcolor}         

\usepackage[numbers,sort&compress]{natbib}
\usepackage{amssymb}
\usepackage{amsmath}
\usepackage{algorithm}
\usepackage{algpseudocode}
\usepackage{caption}
\usepackage{soul}
\usepackage{wrapfig}
\usepackage{makecell}
\usepackage{subcaption}
\usepackage{graphicx}
\usepackage{enumitem}
\usepackage{multirow}
\usepackage{wrapfig}
\usepackage{titlesec}

\titlespacing*{\section}
  {0pt}{*0.4}{*0.4}  
\titlespacing*{\subsection}
  {0pt}{*0.3}{*0.3}

\newcommand{\norm}[1]{\left\lVert#1\right\rVert}
\newcommand{\abs}[1]{\left|#1\right|}
\usepackage{amsthm}
\newtheorem{proposition}{Proposition}
\newtheorem{lemma}{Lemma}

\title{Neural Fields for NV-Center Inverse Sensing}

\author{%
  Zhixuan Zhao$^{*1,2}$\quad
  Tao Zhong\thanks{Equal Contribution.}~~$^1$\quad
  Yixun Hu$^{1}$\\
  \textbf{Nathalie P. de Leon$^{1}$\quad
  Christine Allen-Blanchette$^{1}$} \\
  $^1$Princeton University,\quad  $^2$Tsinghua University\\
  \texttt{\{tzhong, ca15\}@princeton.edu}
}

\begin{document}

\maketitle

\begin{abstract}
    Inverse problems in scientific sensing are often solved with either hand-designed regularizers or supervised networks trained on simulated labels, yet both can fail when the forward model is nonlinear, spectrally coupled, and physically delicate. We study this issue for noise sensing based on nitrogen-vacancy (NV) centers in diamond, where a quantum sensor measures magnetic-noise spectra generated by sparse spin sources. We show that replacing a common scalar/coherent forward approximation with a tensor power-summed dipolar operator changes the inverse landscape and exposes a center-collapse failure mode in free-density optimization. We propose NeTMY, an amortization-free coordinate neural field coupled to the differentiable NV forward model, with annealed positional encoding, multiscale optimization, sparsity/gating, and spectrum-fidelity losses. Across sparse synthetic reconstructions generated by the corrected operator, NeTMY achieves the best localization and distributional metrics in the tested benchmark. Mechanism experiments show that NeTMY does not directly execute the raw density-space gradient; its parameterization smooths and redistributes updates, mitigating the center-collapse pathology. These results position NV quantum sensing as a useful testbed for physics-faithful neural inverse problems.
\end{abstract}
\begin{figure}[h]
    \centering
    \includegraphics[width=\textwidth]{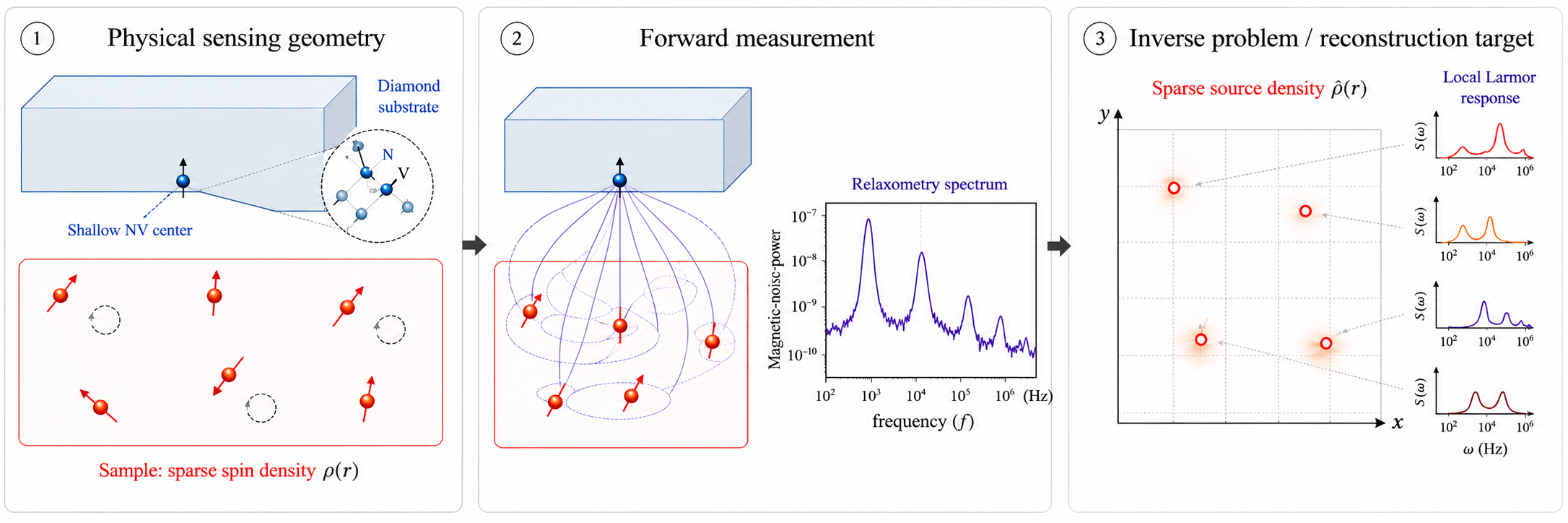}
    \caption{
    NV relaxometry maps sparse fluctuating spins to a noisy frequency-domain measurement.
    The inverse problem aims to recover the spin density and local Larmor response from the spectrum. 
    }
    \label{fig:nv-relaxometry-teaser}
    \vspace{-8pt}
\end{figure}
\section{Introduction}
Scientific machine learning is increasingly being used in settings where the quantity of interest is never observed directly. Instead, one observes the output of an instrument: a low-dimensional, noisy, spatially blurred, and often spectrally transformed measurement generated by a known but imperfect physical process~\cite{arridge2019solving,karniadakis2021physics}. This setting is fundamentally different from standard supervised prediction. Paired labels are expensive or impossible to obtain, synthetic data may not match experimental conditions, and seemingly minor approximations in the forward model can change not only the measurement distribution, but also the geometry of the inverse objective. As a result, a method that performs well under a convenient simulator may fail when the physical operator is made more faithful~\cite{kaipio2007statistical}. This interaction between forward-model fidelity, representation, and optimization is a central challenge for scientific ML.

We study this issue in nitrogen-vacancy (NV) center noise sensing~\cite{doherty2013nitrogen,degen2017quantum}. This problem is closely related to the better-established task of physical reconstruction from static magnetic-field images, where NV magnetic maps have been inverted to recover underlying current or magnetization distributions~\cite{midha2024optimized}. Here, however, the target is not a static field map but a sparse distribution of fluctuating spin sources and their local Larmor response inferred from frequency-dependent magnetic-noise spectra. As illustrated in Figure~\ref{fig:nv-relaxometry-teaser}, in this toy model, NV centers measure spectra induced by nearby fluctuating spins, while the scientific target is the spatial and spectral structure of the spin sources~\cite{casola2018probing,mzyk2022relaxometry,tetienne2013spin}. Recent non-ML NV noise/correlation measurements further motivate this reconstruction setting by showing that multiplexed NV platforms can access noise spectra and magnetic-field correlations at multiple spatial locations~\cite{cheng2025massively}. The resulting inverse problem is highly ill-conditioned~\cite{fuchs2023incomplete}: the dipolar response is low-pass and rapidly decaying~\cite{mzyk2022relaxometry}, nearby sources can merge below the effective point-spread width~\cite{grech2008review}, finite-window effects create spatial bias, and spectrum normalization introduces scale ambiguity and nonlocal gradients~\cite{zhang2018structured}. Thus, even with a differentiable forward model, direct reconstruction from magnetic noise spectra can have misleading descent directions and physically implausible local optima.

Existing approaches leave an important gap. Classical sparse inverse methods such as Tikhonov regularization~\cite{hoerl1970ridge,hoerl1970bridge} or ADMM~\cite{neal2011distributed} optimize the density field directly, but in our setting, this free-density parameterization can follow pathological gradients and collapse to central artifacts under the physical objective. Supervised image-to-image predictors, such as U-Nets~\cite{ronneberger2015u} or GANs~\cite{goodfellow2014generative,radford2015unsupervised}, can be strong when paired training data match the test distribution, but paired spin-density labels are scarce, and synthetic dense training distributions may not generalize to sparse scenes. Generic untrained neural priors~\cite{heckel2018deep,ulyanov2018deep} avoid supervision, but do not by themselves address the tensor relaxometry operator, Larmor-spectrum coupling, or scale ambiguity of this sensing problem.

To this end, we propose NeTMY, an amortization-free neural-field solver for NV relaxometry inversion. NeTMY represents the unknown density and spectral fields with a coordinate MLP and optimizes its parameters separately for each measurement through a differentiable tensor power-summed forward operator. The method uses annealed positional encoding~\cite{tancik2020fourier} and multiscale optimization to recover sparse high-frequency structure while maintaining stable descent, together with spectrum consistency, sparsity, TV regularization~\cite{rudin1992nonlinear}, density gating, and measurement-derived scale correction. Crucially, the forward operator sums tensor-channel magnetic-noise powers rather than squaring a coherently summed scalar field, avoiding nonphysical cross terms that appear in the simplified solver~\cite{guan2024solving,midha2024optimized}. This produces a more faithful inverse objective and exposes failure modes that are hidden or distorted under the scalar approximation.

Our experiments suggest that NeTMY's advantage comes from optimization geometry as much as expressivity. Direct density optimization follows the raw density-space gradient, whereas a neural-field update induces an image-space step filtered by the representation Jacobian. Our contributions are summarized as follows:
\begin{itemize}[leftmargin=*]
    \item We formulate NV noise sensing inversion as a physics-faithful differentiable inverse problem and introduce a tensor power-summed forward operator that avoids nonphysical cross terms in simplified scalar solvers.
    \item We propose NeTMY, an amortization-free coordinate neural-field solver that reconstructs sparse density and spectral fields from individual measurements without paired density labels.
    \item We show empirically and mechanistically that forward-model fidelity changes inverse-problem landscapes, and that NeTMY's parameterization mitigates center-collapse while isolating the contributing components and the remaining failure modes.
\end{itemize}

\section{Related Work}
\label{sec:related}
\textbf{NV centers} have become a standard tool for nanoscale magnetic sensing, mapping local magnetic-noise into optical / $T_1$ and $T_2$ measurement with single-defect resolution~\cite{casola2018probing,degen2017quantum,doherty2013nitrogen,tetienne2013spin,van2015nanometre}. Recent applications include spin-bath sensing~\cite{mzyk2022relaxometry,aharon2019nv,homrighausen2023edge}, nanoscale sensing for condensed matter physics and materials science~\cite{rovny2024nanoscale}, and high-resolution nanoscale measurements~\cite{liu2026high,mamin2013nanoscale,avila2025enhancing,smooha2026sensing,tsukamoto2022accurate,midha2024optimized}. These works establish the sensing physics and forward models but do not study learning-centric inverse algorithms or operator-fidelity-induced optimization-geometry changes, while we treat NV noise sensing as a differentiable inverse-problem benchmark.

\textbf{Sparse optimization algorithms in classical inverse problems}, such as Tikhonov and total-variation regularization~\citep{rudin1992nonlinear,bunks1995multiscale}, ADMM~\citep{neal2011distributed}, and proximal methods~\citep{chambolle2011first,chi2019nonconvex} are standard for ill-posed reconstruction~\citep{arridge2019solving,kaipio2007statistical,virieux2010overview}, with a long line of compressed-sensing~\citep{candes2006robust,donoho2006compressed} and sparse-recovery results~\citep{ahmed2013blind,jagatap2019algorithmic,bora2017compressed,ling2015self,lee2016blind}. These tools have strong guarantees in convex / well-conditioned linear inverse settings, but the NV inverse objective is nonlinear, max-normalized, finite-windowed, tensorial, and spectrally coupled. Standard convexity-based stability analysis does not transfer, and free-density local methods can be attracted to center-biased gradients.

\textbf{Data-driven inverse solvers}, such as supervised image-to-image networks~\citep{ronneberger2015u,mccann2017convolutional}, learned unrolling and primal-dual networks~\citep{monga2021algorithm,hammernik2018learning,yang2018admm,gupta2018cnn}, plug-and-play priors~\citep{venkatakrishnan2013plug,romano2017little,mataev2019deepred,zhang2017learning}, and GAN-based reconstruction~\citep{ongie2020deep,gilton2021deep} provide strong baselines when paired training data matches test data. The field has separately documented reconstruction instabilities under perturbation and model shift~\citep{antun2020instabilities}, robustness gaps~\citep{henderson2018deep,recht2019imagenet}, and benchmarking pitfalls~\citep{liao2021we,raji2021ai,dehghani2021benchmark}. In NV relaxometry, paired spin-density labels are scarce~\cite{mzyk2022relaxometry}. We document train-on-dense / test-on-sparse failure modes and use per-instance physics-driven optimization to side-step them.

\textbf{Untrained neural priors, neural fields, and physics-informed differentiable scientific ML.}
On the parameterization-as-prior side, Deep Image Prior~\citep{ulyanov2018deep}, Deep Decoder~\citep{heckel2018deep,heckel2020compressive}, Fourier features / SIREN~\citep{tancik2020fourier,sitzmann2020implicit}, NeRF~\citep{mildenhall2021nerf}, INR-based reconstruction~\citep{park2021nerfies,martel2021acorn,molaei2023implicit,shi2024improved,reed2021dynamic,shen2022nerp,alkhouri2025understanding,habring2024neural,zhao2024single}, and 3D Gaussian Splatting~\citep{kerbl20233d} demonstrate that the structure of a coordinate or generator network itself biases reconstruction. On the physics-in-loop side, physics-informed neural networks~\citep{raissi2019physics,karniadakis2021physics,cuomo2022scientific,inda2022physics} and differentiable simulators~\citep{zhong2026neural,zhao2025revealing} incorporate known physical laws as losses or as differentiable forward models. Generic neural priors do not encode NV tensor physics, Larmor-spectrum coupling, or scale ambiguity, and the physics-informed thread has rarely studied how forward-operator fidelity reshapes inverse-problem optimization geometry. We connect the two threads on a concrete differentiable scientific inverse problem and ground the parameterization-geometry effect in measurable optimization signatures.

\section{Preliminaries and Problem Formulation}
\label{sec:prelim}

We fix notation, the forward measurement model, the canonical inverse objective, and the structural sources of ill-posedness that the Method (Section~\ref{sec:method}) is designed to address. Detailed derivations and diagnostics are deferred to Appendices~\ref{app:operator-derivation}, \ref{app:scale-correction}, and \ref{app:illposed-diagnostics}.

\subsection{Measurement Model}
\label{sec:meas-model}
Let $\Omega\subset\mathbb{R}^{2}$ be a finite spatial grid and let $\rho:\Omega\to\mathbb{R}_{\geq 0}$ denote the unknown spin-source density. A widefield NV array sits at standoff height $z_{0}$ above $\Omega$ and reads out a frequency-dependent noise spectrum $S_{\mathrm{obs}}(\omega,\mathbf{r})=\mathcal{F}(\rho,\omega_{L})(\omega,\mathbf{r})+\varepsilon$ on a discrete grid $W\subset\mathbb{R}_{+}$, where $\omega_{L}:\Omega\to\mathbb{R}_{+}$ is a per-pixel Larmor field encoding local detuning, $\varepsilon\sim\mathcal{N}(0,\sigma_{\varepsilon}^{2})$ is sensor noise with standard deviation $\sigma_{\varepsilon}$, and $\mathcal{F}$ is a differentiable forward operator. The dipolar Green tensor coupling sample spins to the NV $z$-axis is
\begin{equation}
\label{eq:green}
G_{ia}(\mathbf{R}) \;=\; \frac{\mu_{0}}{4\pi}\,\frac{3R_{i}R_{a}-\norm{\mathbf{R}}^{2}\delta_{ia}}{\norm{\mathbf{R}}^{5}},\qquad i,a\in\{x,y,z\},
\end{equation}
where $\mathbf{R}(\mathbf{r},\mathbf{r}_{\mathrm{src}})=(\mathbf{r}-\mathbf{r}_{\mathrm{src}},z_{0})$, $\mu_{0}$ is the vacuum permeability and $\delta_{ia}$ the Kronecker delta. The spectral response of a dipole resonant at $\omega_{L}$ with linewidth $\gamma$ is the Lorentzian $L(\omega;\omega_{L})=\gamma^{2}/((\omega-\omega_{L})^{2}+\gamma^{2})$~\citep{tetienne2013spin,doherty2013nitrogen,casola2018probing,mzyk2022relaxometry}.

We compare two FFT-based forward operators that differ in how transverse-field noise is aggregated across the channels of $G$ and across distinct sources:
\begin{align}
\label{eq:s1s2}
\mathcal{F}_{1}(\rho,\omega_{L})(\omega,\mathbf{r})
&\;=\;\bigl|(G_{\mathrm{nv}}\!\ast\rho)(\mathbf{r})\bigr|^{2}\,L\bigl(\omega;\omega_{L}(\mathbf{r})\bigr),\\
\mathcal{F}_{2}(\rho,\omega_{L})(\omega,\mathbf{r})
&\;=\;\biggl[\sum_{a\in\{x,y,z\}}\!\bigl(\abs{G_{az}}^{2}\!\ast\rho\bigr)(\mathbf{r})\biggr]\,L\bigl(\omega;\omega_{L}(\mathbf{r})\bigr).\nonumber
\end{align}
where $\ast$ denotes 2D spatial convolution. The scalar/coherent operator $\mathcal{F}_{1}$ squares a single NV-axis projection $G_{\mathrm{nv}}=\sum_{a}n_{a}G_{az}$ ($\mathbf{n}=(0,0,1)$ here) after spatial superposition, while the tensor/incoherent operator $\mathcal{F}_{2}$ sums per-channel power kernels and treats distinct sources as independent thermal fluctuators. Their pointwise difference is a kernel-product cross-term coupling pairs of distinct sources, which is the algebraic origin of the inverse-landscape difference exploited in Section~\ref{sec:experiments}. We treat $\mathcal{F}_{2}$ as the physically appropriate inversion operator. Because $\mathcal{F}_{2}$ factorizes the Lorentzian response outside the
kernel convolution, it is exact for locally constant $\omega_L$ and remains
controlled when
$\chi\sim\Delta\omega_L^{(\mathrm{kernel})}/\gamma\ll1$, which is suitable for the research scale. The scope of application and limitations of $\mathcal{F}_{2}$ is in Appendix~\ref{app:operator-lorentzian_robustness}; we additionally introduce a slower direct simulator $\mathcal{F}_{3}$ that evaluates the source-side superposition without FFT factorization as the most physically accurate operator(Appendix~\ref{app:operator-direct}). To avoid the inverse-crime risk of reusing the inversion operator as the data-generation operator~\citep{kaipio2007statistical}, the main-text benchmark generates spectra with $\mathcal{F}_{3}$ and inverts under $\mathcal{F}_{1}$ or $\mathcal{F}_{2}$; matched-operator results that generate and invert with the same $\mathcal{F}_{1}$ or $\mathcal{F}_{2}$ are reported in the appendix as a complementary view. Full pointwise derivations of $\mathcal{F}_{1}$, $\mathcal{F}_{2}$, $\mathcal{F}_{3}$, and the regime of physical applicability of each form are given in Appendix~\ref{app:operator-derivation}.

\subsection{Inverse Problem}
\label{sec:inverse-problem}
Following standard practice in NV sensing~\citep{kumar2024room,midha2024optimized,casola2018probing}, the spectrum is summed across $W$ and max-normalized before being compared on a logarithmic scale. Writing $N(\mathbf{r};S)=\sum_{\omega\in W}S(\omega,\mathbf{r})$ and $\widehat{N}(\mathbf{r};S)=N(\mathbf{r};S)/\max_{\mathbf{r}'}N(\mathbf{r}';S)$, the data-fidelity term $\mathcal{D}(\mathcal{F}(\rho,\omega_{L}),S_{\mathrm{obs}})$ is the pixelwise log-MSE between $\widehat{N}(\,\cdot\,;\mathcal{F}(\rho,\omega_{L}))$ and $\widehat{N}(\,\cdot\,;S_{\mathrm{obs}})$ (full form in Appendix~\ref{app:scale-correction}). The canonical regularized inverse problem is
\begin{equation}
\label{eq:invprob}
\min_{\rho\geq 0,\;\omega_{L}}\;
\mathcal{D}\bigl(\mathcal{F}(\rho,\omega_{L}),S_{\mathrm{obs}}\bigr)
\;+\;\lambda_{1}\norm{\rho}_{1}
\;+\;\lambda_{\mathrm{TV}}\,\mathrm{TV}(\rho),
\end{equation}
where $\norm{\rho}_{1}=\sum_{\mathbf{r}\in\Omega}\abs{\rho(\mathbf{r})}$ is the $\ell_{1}$ norm of the discretized density, $\mathrm{TV}(\rho)=\sum_{\mathbf{r}\in\Omega}\bigl(\abs{\rho(\mathbf{r}+\mathbf{e}_{x})-\rho(\mathbf{r})}+\abs{\rho(\mathbf{r}+\mathbf{e}_{y})-\rho(\mathbf{r})}\bigr)$ is the anisotropic total variation, $\mathbf{e}_{x},\mathbf{e}_{y}$ are unit grid steps, and $\lambda_{1},\lambda_{\mathrm{TV}}\geq 0$ are scalar weights. Eq.~\eqref{eq:invprob} is the objective shared by every method in our benchmark; classical baselines (Tikhonov, ADMM, GaussianSplat) and the proposed neural field method differ in (i) how $\rho$ is parameterized and (ii) which method-specific physics losses they add (Section~\ref{sec:method}).

A direct consequence of the max-normalization in $\widehat{N}$ is that the linearity $\mathcal{F}_{2}(c\rho,\omega_{L})=c\,\mathcal{F}_{2}(\rho,\omega_{L})$ leaves $\widehat{N}$ invariant for every $c>0$, so the pre-normalization scale of $\rho$ is unidentifiable from $\widehat{N}$ alone~\citep{ahmed2013blind,ling2015self}. This is a property of the problem; the resolution adopted by the methods compared here is described in Section~\ref{sec:method} and proved in Appendix~\ref{app:scale-correction}.

\subsection{Ill-posedness}
\label{sec:illposed}
Even with $\mathcal{F}_{2}$ in place, four structural properties make Eq.~\eqref{eq:invprob} severely ill-conditioned. We name and locate them here; the formal Lemma~\ref{lem:freqdecay} (frequency decay), the per-pixel Jacobian-column squared norm, the max-normalization gradient, and a heuristic estimate of the effective point-spread width are in Appendix~\ref{app:illposed-diagnostics}.

\textbf{(P1) Exponential frequency suppression.} Each tensor channel $|G_{az}|^{2}$ in $\mathcal{F}_{2}$ has 2D Fourier transform with exponential envelope $e^{-k z_{0}}$ up to algebraic factors in $k$, where $k=\norm{\mathbf{k}}$ is the magnitude of the spatial-frequency vector $\mathbf{k}$~\citep{van2015nanometre}, so high-spatial-frequency density features carry vanishingly little forward signal (Lemma~\ref{lem:freqdecay}).

\textbf{(P2) Finite-window center bias.} The convolutional footprint of a source pixel is more fully observable at the window center than at the boundary, so the per-pixel forward sensitivity $\norm{\partial \mathcal{F}/\partial\rho(\mathbf{r}_{0})}$ is larger near the center, producing a centered descent direction even from a uniform initialization.

\textbf{(P3) Max-normalization peak coupling.} The derivative of the max-normalized noise map contains a nonlocal term concentrated at the current peak pixel, sharpening any incipient peak and amplifying the (P2) center bias once a peak appears.

\textbf{(P4) Resolution-limited merging and joint $(\rho,\omega_{L})$ ambiguity.} Source pairs separated by less than $\mathcal{O}(z_{0})$ cannot be disambiguated, and $\omega_{L}$ is identifiable only on the support of $\rho$.

These pathologies, together with $\mathcal{F}_{2}$'s spectrally coupled, finite-window structure, are why the parameterization of $\rho$ matters more than ordinary regularizer tuning in Eq.~\eqref{eq:invprob}. Section~\ref{sec:method} addresses (P1)--(P3) through representation geometry and the curriculum of the neural-field optimizer, and (P4) through density-masked Larmor handling.

\section{Method}
\label{sec:method}

NeTMY solves the inverse problem of Eq.~\eqref{eq:invprob} \emph{per measurement}: rather than learning an amortized inverse map from a paired dataset, it represents the unknown density $\rho$ and Larmor field $\omega_{L}$ as the output of a coordinate MLP and optimizes the MLP parameters against a single observed spectrum through a differentiable forward operator $\mathcal{F}\in\{\mathcal{F}_{1},\mathcal{F}_{2}\}$. We then introduce design choices that, in combination, address the structural pathologies (P1)--(P4) of \S~\ref{sec:illposed} through the optimization geometry of the parameterization rather than through additional regularizer tuning. A method overview is given in Figure~\ref{fig:method-overview}; full architecture, loss, schedule, and Jacobian-filtering derivations are in Appendix~\ref{app:method}.

\begin{figure}[t]
    \centering
    \includegraphics[width=\textwidth]{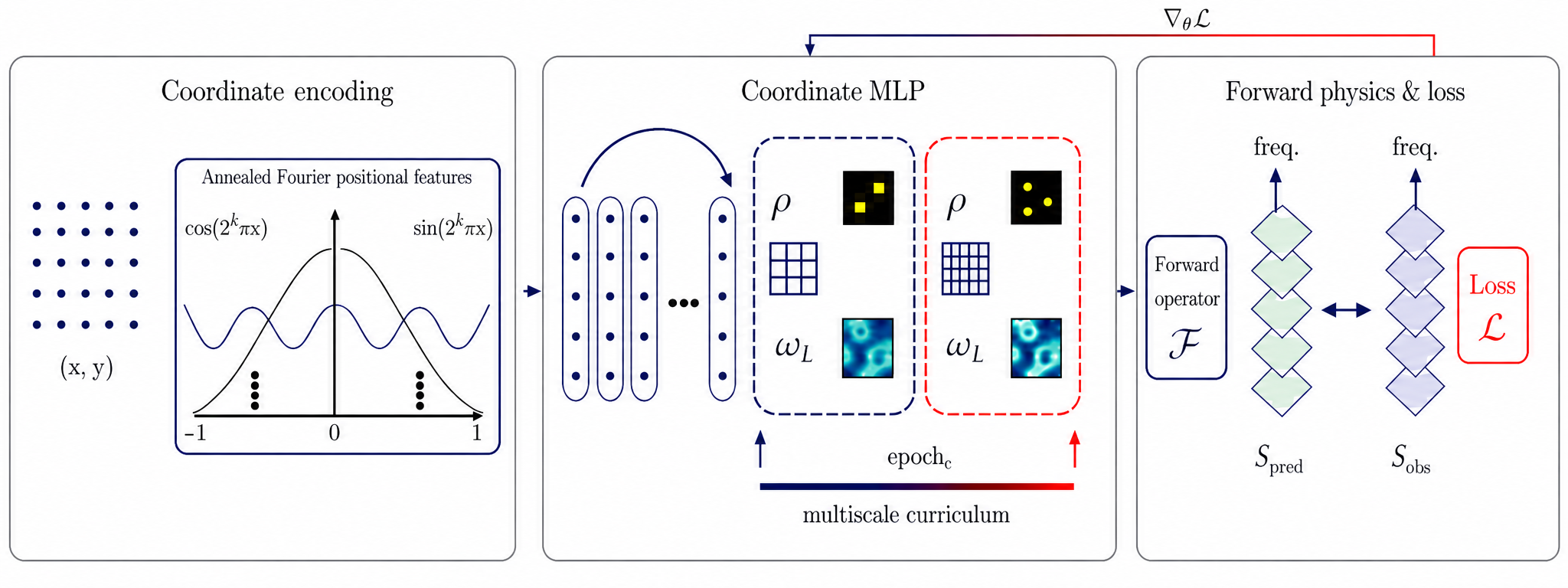}
    \caption{NeTMY pipeline. Coordinates pass through annealed Fourier features and a coordinate MLP $f_{\theta}$ to a density $\rho$ and Larmor field $\omega_{L}$; a differentiable forward operator $\mathcal{F}$ predicts a spectrum compared with the observation under loss $\mathcal{L}$. The same MLP is trained per instance across two resolutions in sequence; gradients flow back through the operator, with no paired training data.}
    \label{fig:method-overview}
    \vspace{-10pt}
\end{figure}

\subsection{Coordinate Neural Field}
\label{sec:method-field}
We parameterize the unknowns as an MLP of grid coordinates,
\begin{equation}
\label{eq:method-field}
f_{\theta}:(x,y)\;\mapsto\;\bigl(\rho_{\theta}(x,y),\;\omega_{L,\theta}(x,y)\bigr),
\qquad
(x,y)\in\Omega,
\end{equation}
with the network input augmented by annealed Fourier positional features $\gamma_{\beta}(x,y)$ whose high-frequency bands are gated on smoothly through a training-progress scalar $\beta$~\citep{tancik2020fourier,park2021nerfies,sitzmann2020implicit}. Annealing addresses (P1) by fitting low-$\norm{\mathbf{k}}$ structure of $\rho$ before high-$\norm{\mathbf{k}}$ detail. The MLP is a fully connected backbone with a skip connection re-injecting the encoded input partway through.

The two output heads use bounded transforms tailored to the physical priors. The density head is a gated softplus,
\begin{equation}
\label{eq:method-density-head}
\rho_{\theta}(x,y)\;=\;\mathrm{softplus}\bigl(h_{\rho}(x,y)\bigr)\,\sigma\bigl(g(x,y)\bigr),
\end{equation}
where $h_{\rho}$ is the raw density logit, $g$ is a separate gate logit, and $\sigma$ is the elementwise sigmoid. The softplus enforces $\rho_{\theta}\geq 0$, while the gate factor $\sigma(g)\in(0,1)$ lets the network drive pixel regions to near zero without saturating, which is the operational counterpart of the (P3) self-reinforcing peak observation. The Larmor head sigmoid-maps a logit into the device-determined band $[\omega_{L,\min},\omega_{L,\max}]$ and is masked to the support of $\rho$ so that gradient flows into $\omega_{L}$ only where the data constrains it (P4).

\subsection{Multiscale Curriculum}
\label{sec:method-multiscale}
Optimization proceeds in two stages on the same MLP parameters, with the coordinate grid resolution increasing between stages. Stage~1 fits coarse spatial structure under the coarsened forward operator at the base learning rate; Stage~2 refines high-frequency detail at the higher resolution at a reduced learning rate. By the bandlimit argument of Lemma~\ref{lem:freqdecay}, low-$\norm{\mathbf{k}}$ modes carry the bulk of the forward signal and are well-conditioned, so fitting them first delivers a coarse minimum closer to the true support before the high-frequency curriculum activates. Per-stage resolutions, epoch counts, and learning rates are reported in Appendix~\ref{app:method-multiscale} (Table~\ref{tab:app-method-schedule}).

\subsection{Per-Measurement Objective}
\label{sec:method-loss}
NeTMY minimizes a NeTMY-specific extension of Eq.~\eqref{eq:invprob} by gradient descent on $\theta$ alone, with no oracle access to $\rho_{\star}$ or $\omega_{L,\star}$. The objective combines the canonical fidelity $\mathcal{D}$ with two method-specific physics losses motivated by the ill-posedness analysis,
\begin{equation}
\label{eq:method-loss}
\mathcal{L}(\theta)\;=\;\mathcal{D}\bigl(\mathcal{F}(\rho_{\theta},\omega_{L,\theta}),S_{\mathrm{obs}}\bigr)\;+\;\lambda_{1}\norm{\rho_{\theta}}_{1}\;+\;\lambda_{\mathrm{TV}}\,\mathrm{TV}(\rho_{\theta})\;+\;\lambda_{N}\,\mathcal{R}_{\mathrm{nm}}\;+\;\lambda_{\rho}\,\mathcal{R}_{\mathrm{ds}},
\end{equation}
where $\mathcal{R}_{\mathrm{nm}}$ is a mean-normalized noise-map loss (a scale-stable companion to $\mathcal{D}$ that anchors gradients on the support; full form in Appendix~\ref{app:method-losses}) and $\mathcal{R}_{\mathrm{ds}}$ is a direct-density loss matching $\rho_{\theta}^{2}$ to the mean-normalized observed noise map (a soft, oracle-free support proxy whose theoretical scope is discussed in Appendix~\ref{app:method-losses}). The five loss-term weights are fixed across all reported runs (Appendix~\ref{app:method-losses}, Table~\ref{tab:app-method-losses}); the relative weight of $\mathcal{D}$ vs.\ $\mathcal{R}_{\mathrm{nm}}$ is rebalanced between Stage~1 and Stage~2 to reflect the different roles of log-MSE support discovery (coarse stage) and mean-normalized refinement (fine stage), as detailed in Appendix~\ref{app:method-losses}. Identifiability of $\omega_{L}$ on the support of $\rho$ (P4) is enforced architecturally: the Larmor head is multiplied by the hard support mask $\mathbf{1}\{\rho_{\theta}(\mathbf{r})>\tau\,\max_{\mathbf{r}'}\rho_{\theta}(\mathbf{r}')\}$, with the mask treated as a stop-gradient constant during back-prop, so the nondifferentiable indicator is not differentiated and the Larmor logit receives gradients only on predicted support (Appendix~\ref{app:method-arch}, Eq.~\eqref{eq:app-larmor-mask}). The loss weights $\lambda_{1},\lambda_{\mathrm{TV}}\geq 0$ are the standard $\ell_{1}$ and TV weights of Eq.~\eqref{eq:invprob}, and $\lambda_{N},\lambda_{\rho}\geq 0$ are the NeTMY-specific weights for $\mathcal{R}_{\mathrm{nm}},\mathcal{R}_{\mathrm{ds}}$ respectively.

\textbf{Energy-anchored scale correction.} The pre-normalization scale of $\rho$ is unidentifiable from $\widehat{N}$ (Section~\ref{sec:inverse-problem}; Proposition~\ref{prop:scale}). After the per-measurement optimization terminates, NeTMY rescales the predicted density by the energy ratio $\widehat{\rho}\leftarrow\alpha\widehat{\rho}$, $\alpha=E_{\mathrm{obs}}/E_{\mathrm{pred}}$ for $\mathcal{F}_2$, where $E_{\mathrm{obs}}=\sum_{\omega,\mathbf{r}}S_{\mathrm{obs}}$ and $E_{\mathrm{pred}}=\sum_{\omega,\mathbf{r}}\mathcal{F}(\widehat{\rho},\omega_{L,\theta})$. We use the operator-matched convention in the main-body tables; This correction mainly affects density MSE; the primary localization-oriented metrics
are insensitive to the absolute density scale. The physically matched closed-form correction for the quadratic operator $\mathcal{F}_{1}$ is $\sqrt{E_{\mathrm{obs}}/E_{\mathrm{pred}}}$ and is derived alongside the linear-$\mathcal{F}_{2}$ form in Appendix~\ref{app:scale-correction} (Eq.~\eqref{eq:app-c1}). Operating this as a one-shot post-correction rather than a soft penalty inside Eq.~\eqref{eq:method-loss} avoids reweighting the data-fidelity gradient and gives identical post-processing across baselines (Appendix~\ref{app:method-scale}).

\subsection{Optimization Geometry: a Filtering View of the Update}
\label{sec:method-geometry}
Free-density solvers (Tikhonov, ADMM) minimize Eq.~\eqref{eq:invprob} by stepping $\rho$ directly, $\rho_{t+1}=\rho_{t}-\eta\,\nabla_{\rho}\mathcal{L}(\rho_{t})$, and therefore execute the raw density-space gradient that inherits the (P2) center bias and the (P3) max-normalization peak coupling. NeTMY instead steps the parameters $\theta$, with chain rule
\begin{equation}
\label{eq:method-geometry}
\Delta\rho\;\approx\;J_{\theta}\,\Delta\theta\;=\;-\eta\,J_{\theta}J_{\theta}^{\top}\,\nabla_{\rho}\mathcal{L}\;=\;-\eta\,G_{\theta}\,\nabla_{\rho}\mathcal{L},
\qquad J_{\theta}\;=\;\partial f_{\theta}/\partial\theta,
\end{equation}
so the realized image-space update is the raw density gradient \emph{filtered} by the positive semidefinite kernel $G_{\theta}=J_{\theta}J_{\theta}^{\top}\in\mathbb{R}^{|\Omega|\times|\Omega|}$ induced by the parameterization (Appendix~\ref{app:method-geometry}). For a coordinate MLP with annealed Fourier features, $G_{\theta}$ is a smoothing, structurally coupled operator across pixels: a sharp center spike in $\nabla_{\rho}\mathcal{L}$ is redistributed by $G_{\theta}$ rather than executed as a singular update, which is precisely the geometry under which the iter-$0$ centered gradient documented in (P2) does not need to be amplified by (P3). This view connects the architectural choices of Sections~\ref{sec:method-field}--\ref{sec:method-multiscale} to the empirical center-collapse / energy-barrier mechanism diagnosis in Section~\ref{sec:experiments}: NeTMY does not eliminate the pathological gradient but bends its realization in image space.

A complete training algorithm, the explicit functional forms of $\mathcal{R}_{\mathrm{nm}}$ and $\mathcal{R}_{\mathrm{ds}}$, the architecture and PE-annealing schedule, and the formal version of Eq.~\eqref{eq:method-geometry} are given in Appendix~\ref{app:method}.

\section{Experiments}
\label{sec:experiments}

We evaluate NeTMY against three families of inverse solvers under the formulation of Eq.~\eqref{eq:invprob} and the ill-posedness pathologies of Section~\ref{sec:illposed}, organized around four questions: (i) does the choice of inversion operator $\mathcal{F}\in\{\mathcal{F}_{1},\mathcal{F}_{2}\}$ change which methods recover sparse spin sources reliably (Section~\ref{sec:exp-benchmark}); (ii) what is the optimization-geometry mechanism behind the ranking (Section~\ref{sec:exp-mechanism}); (iii) which components of the NeTMY design carry the gain (Section~\ref{sec:exp-ablation}); and (iv) does the operator-fidelity gap survive on a real NV-relaxometry dataset where the ground-truth spin source is independently constrained (Section~\ref{sec:exp-realdata}). Implementation details, full per-method per-sample results, hyperparameter sweeps, the supervised-baseline distribution-shift study, and runtime statistics are in Appendix~\ref{app:experiments}.

\subsection{Setup}
\label{sec:exp-setup}

\textbf{Datasets.} Two regimes in the main text (full construction in Appendix~\ref{app:exp-datasets}). (i) A \emph{cross-fidelity benchmark} of $512$ synthetic samples in which spectra are generated by the source-side direct simulator $\mathcal{F}_{3}$ (Eq.~\eqref{eq:app-s3}) and methods invert under either $\mathcal{F}_{1}$ or $\mathcal{F}_{2}$, avoiding the inverse-crime risk of reusing the inversion operator as the data-generation operator~\citep{kaipio2007statistical}. The samples cover eight scene classes (few/medium/many sources $\times$ close/medium/far minimum separations relative to the standoff $z_{0}$). (ii) A \emph{real-data cross-check} on $\alpha$-RuCl$_{3}$ (8 NVs, \citet{kumar2024room}) with an independent SRIM depth prior $14\pm 5$\,nm. Two complementary matched-operator benchmarks (data generation and inversion by the same $\mathcal{F}_{1}$ or $\mathcal{F}_{2}$, $512$ samples each, with broader baseline coverage and density MSE only) are reported alongside in Table~\ref{tab:matched} and in Appendix~\ref{app:exp-matched}.

\textbf{Baselines.} Four reconstruction families that share Eq.~\eqref{eq:invprob} but differ in how $\rho$ is parameterized: Tikhonov (free-density gradient descent with $\ell_{2}$/TV)~\citep{hoerl1970ridge,rudin1992nonlinear}; ADMM (variable-splitting with augmented Lagrangian and box constraints)~\citep{neal2011distributed}; GaussianSplat (explicit Gaussian-splat primitives with prune/split/clone/merge)~\citep{kerbl20233d}; and NeTMY (this work). L-BFGS~\citep{liu1989limited} is included only in the mechanism analysis. Untrained-prior (DeepDecoder~\citep{heckel2018deep}) and supervised baselines (U-Net~\citep{ronneberger2015u}, GAN~\citep{goodfellow2014generative}, HybridNeTMY) are reported in Appendix~\ref{app:exp-supervised}.

\textbf{Metrics.} Three primary metrics that emphasize sparse-localization and distributional fidelity over pixelwise mean-square error: GMSD (gradient-magnitude similarity deviation)~\citep{xue2013gradient}; Hungarian F1 (one-to-one matched localization at fixed radius)~\citep{kuhn1955hungarian}; and Sliced Wasserstein distance (SWD) over $128$ random projections~\citep{bonneel2015sliced}. Pixelwise density MSE is a secondary metric included for comparability with prior NV-relaxometry work~\citep{midha2024optimized,kumar2024room}; on sparse scenes, it is dominated by background and undersells localization quality. Definitions and matching tolerances are in Appendix~\ref{app:exp-metrics}.

\subsection{Operator Fidelity Reshapes the Reconstruction Benchmark}
\label{sec:exp-benchmark}

Tables~\ref{tab:main} and~\ref{tab:matched} report the cross-fidelity and matched-operator benchmarks; we highlight 3 observations.

\begin{table}[t]
\centering
\scriptsize
\setlength{\tabcolsep}{2.2pt}
\renewcommand{\arraystretch}{0.94}
\caption{Cross-fidelity benchmark on $\mathcal{F}_{3}$-generated samples; methods invert under $\mathcal{F}_{1}$ or $\mathcal{F}_{2}$. Best per metric in bold. We report the mean and 95\% CI across 3 seeds.}
\label{tab:main}

\resizebox{\textwidth}{!}{%
\begin{tabular}{lcccc|cccc}
\toprule
& \multicolumn{4}{c}{Inversion under $\mathcal{F}_{1}$ (scalar)}
& \multicolumn{4}{c}{Inversion under $\mathcal{F}_{2}$ (tensor)} \\
\cmidrule(lr){2-5}\cmidrule(lr){6-9}
Method 
& GMSD$\downarrow$ 
& F1$\uparrow$ 
& SWD$\downarrow$ 
& MSE$\downarrow$ 
& GMSD$\downarrow$ 
& F1$\uparrow$ 
& SWD$\downarrow$ 
& MSE$\downarrow$ \\
\midrule

\multicolumn{9}{l}{\textit{Classical regularized solvers}} \\
Tikhonov 
& $0.178{\pm}0.041$ 
& $0.465{\pm}0.436$ 
& $0.119{\pm}0.043$ 
& $0.0119{\pm}0.0155$ 
& $0.202{\pm}0.043$ 
& $0.762{\pm}0.318$ 
& $0.100{\pm}0.041$ 
& $0.0082{\pm}0.0118$ \\

ADMM     
& $0.235{\pm}0.088$ 
& $0.267{\pm}0.389$ 
& $0.052{\pm}0.028$ 
& $0.0183{\pm}0.0177$ 
& $0.197{\pm}0.061$ 
& $0.084{\pm}0.208$ 
& $0.134{\pm}0.048$ 
& $3.8364{\pm}4.9726$ \\

\midrule
\multicolumn{9}{l}{\textit{Parameterized solvers}} \\

GaussianSplat   
& $\mathbf{0.144{\pm}0.082}$ 
& $\mathbf{0.816{\pm}0.213}$ 
& $\mathbf{0.044{\pm}0.037}$ 
& $0.0076{\pm}0.0122$ 
& $0.140{\pm}0.070$ 
& $0.740{\pm}0.241$ 
& $0.061{\pm}0.039$ 
& $0.0117{\pm}0.0171$ \\

\textbf{NeTMY (Ours)} 
& $0.189{\pm}0.116$ 
& $0.733{\pm}0.334$ 
& $0.051{\pm}0.048$ 
& $\mathbf{0.0066{\pm}0.0080}$ 
& $\mathbf{0.125{\pm}0.058}$ 
& $\mathbf{0.882{\pm}0.102}$ 
& $\mathbf{0.036{\pm}0.023}$ 
& $\mathbf{0.0078{\pm}0.0082}$ \\

\bottomrule
\end{tabular}%
}
\vspace{-12pt}
\end{table}

\begin{table}[t]
\small
\begin{minipage}[t]{0.59\textwidth}
\centering
\setlength{\tabcolsep}{4pt}
\caption{Matched-operator benchmarks. The matched operator regime is operator-dependent.} 
\label{tab:matched}
\begin{tabular}{lcc|cc}
\toprule
& \multicolumn{2}{c}{$\mathcal{F}_{1}/\mathcal{F}_{1}$} & \multicolumn{2}{c}{$\mathcal{F}_{2}/\mathcal{F}_{2}$} \\
\cmidrule(lr){2-3}\cmidrule(lr){4-5}
Method & MSE$\downarrow$ & NoiseMSE$\downarrow$ & MSE$\downarrow$ & NoiseMSE$\downarrow$ \\
\midrule
Tikhonov     & 0.015 & 2.25 & 0.012 & 0.81 \\
ADMM         & \textbf{0.006} & \textbf{0.32} & 0.256 & 22.9 \\
GaussianSplat       & 0.010 & 0.35 & 0.007 & 1.09 \\
HybridNeTMY  & 0.025 & 6.47 & 1.011 & 0.23 \\
DeepDecoder  & 1.61   & 31.38  & 0.256 & 24.3 \\
UntrainedGAN & 0.026 & 43.0 & 0.026   &  24.2 \\
\textbf{NeTMY (Ours)} & 0.009 & 1.15 & \textbf{0.004} & \textbf{0.21} \\
\bottomrule
\end{tabular}
\end{minipage}\hfill
\begin{minipage}[t]{0.39\textwidth}
\centering
\setlength{\tabcolsep}{4pt}
\caption{Cumulative ablation. Each row removes one component cumulatively.}
\label{tab:ablation}
\begin{tabular}{lcc}
\toprule
Configuration & MSE$\downarrow$ & SSIM$\uparrow$ \\
\midrule
\textbf{Full model}     & \textbf{.0039} & \textbf{.913} \\
$-$TV               & .0042 & .893 \\
$-$annealed PE      & .0086 & .897 \\
$-$PE               & .0091 & .895 \\
$-$multiscale       & .0103 & .890 \\
$-\ell_{1}$         & .0104 & .880 \\
$-$gate             & .0118 & .876 \\
$-\mathcal{R}_{\mathrm{ds}}$    & .0116 & .890 \\
\bottomrule
\end{tabular}
\end{minipage}
\vspace{-12pt}
\end{table}

\textbf{(i) The physics-corrected operator changes the performance of the methods.} Moving from $\mathcal{F}_{1}$ to $\mathcal{F}_{2}$ inversion lifts Hungarian F1 and lowers SWD for two of four; the largest gain is NeTMY (Hungarian-F1: $0.733\!\to\!0.882$; SWD: $0.051\!\to\!0.036$). ADMM's density-MSE \emph{significantly worsens} under $\mathcal{F}_{2}$, because the more faithful operator sharpens the centered collapse of Section~\ref{sec:exp-mechanism}.

\textbf{(ii) Under $\mathcal{F}_{2}$, NeTMY tops every primary metric.} NeTMY beats GaussianSplat by $19\%$ on F1 and $41\%$ on SWD (Table~\ref{tab:main}); the third-decimal density-MSE gap reflects MSE being background-dominated on sparse scenes, while the localization-aware metrics measure what the design targets.

\textbf{(iii) The matched-operator regime is operator-dependent.} On the matched-operator benchmark (Table~\ref{tab:matched}), ADMM is the lowest-MSE method under $\mathcal{F}_{1}/\mathcal{F}_{1}$ (no centering pathology) but degrades by two orders of magnitude under $\mathcal{F}_{2}/\mathcal{F}_{2}$, while NeTMY is the lowest-MSE method on all $\mathcal{F}_{2}/\mathcal{F}_{2}$ metrics; the full $\mathcal{F}_{2}/\mathcal{F}_{2}$ leaderboard with secondary metrics and runtime is in Appendix~\ref{app:exp-matched}.

\begin{figure}[t]
    \centering
    \includegraphics[width=\textwidth]{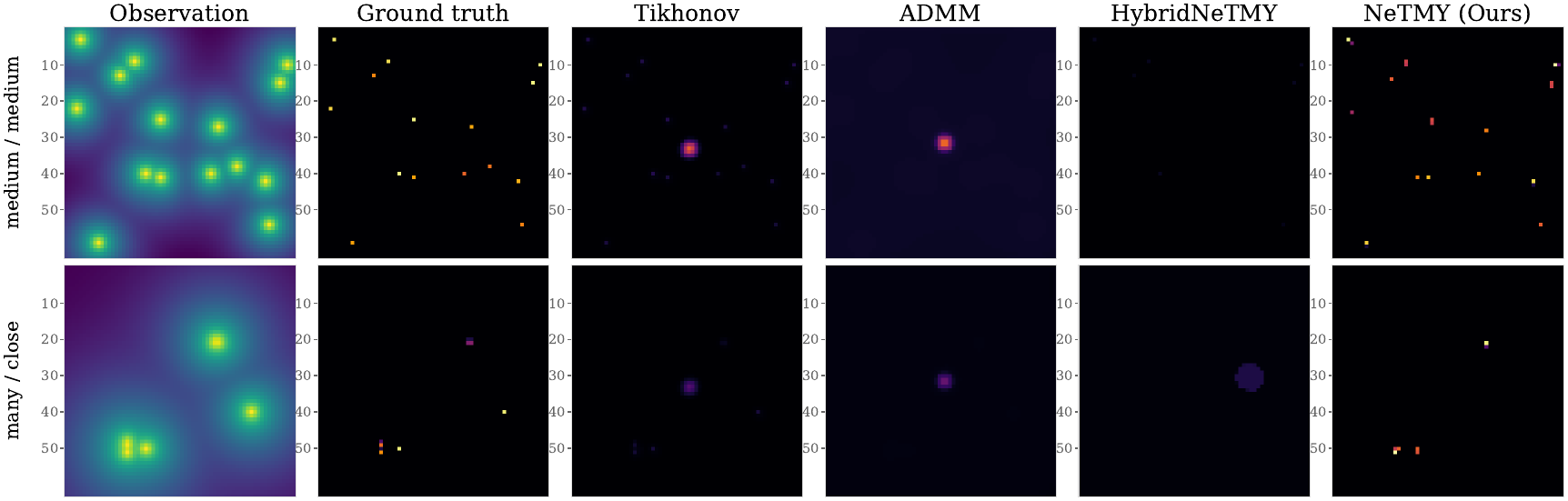}
    \caption{Qualitative reconstructions on two $\mathcal{F}_{3}$-generated samples inverted under $\mathcal{F}_{2}$. Free-density solvers center-collapse; HybridNeTMY and NeTMY preserve off-center structure.}
    \label{fig:qualitative-main}
    \vspace{-13pt}
\end{figure}

Figure~\ref{fig:qualitative-main} illustrates the pixel-level pattern: free-density solvers execute the centered iter-0 gradient of (P2)--(P3) and end at a centrally collapsed minimum, while NeTMY preserves off-center structure and reconstructs finer details. The matched-operator results in Table~\ref{tab:matched} reuse the inversion operator as the data-generation operator and are therefore subject to the inverse-crime caveat~\citep{kaipio2007statistical}. More qualitative examples are in Appendix~\ref{app:exp-examples}. We use Table~\ref{tab:main} as primary evidence and the real-data check (Section~\ref{sec:exp-realdata}) as a simulator-free test.

\subsection{Optimization Geometry: Why Free-Density Solvers Center-Collapse at an Early Stage}
\label{sec:exp-mechanism}

Lemma~\ref{lem:app-filter} predicts that the iter-$0$ centered gradient produced by (P2)--(P3) is executed verbatim by free-density solvers and filtered by the parameterization Jacobian for NeTMY. We measure three quantities on representative many-source samples (Figure~\ref{fig:mechanism}).

\begin{figure}[t]
    \centering
    \includegraphics[width=\textwidth]{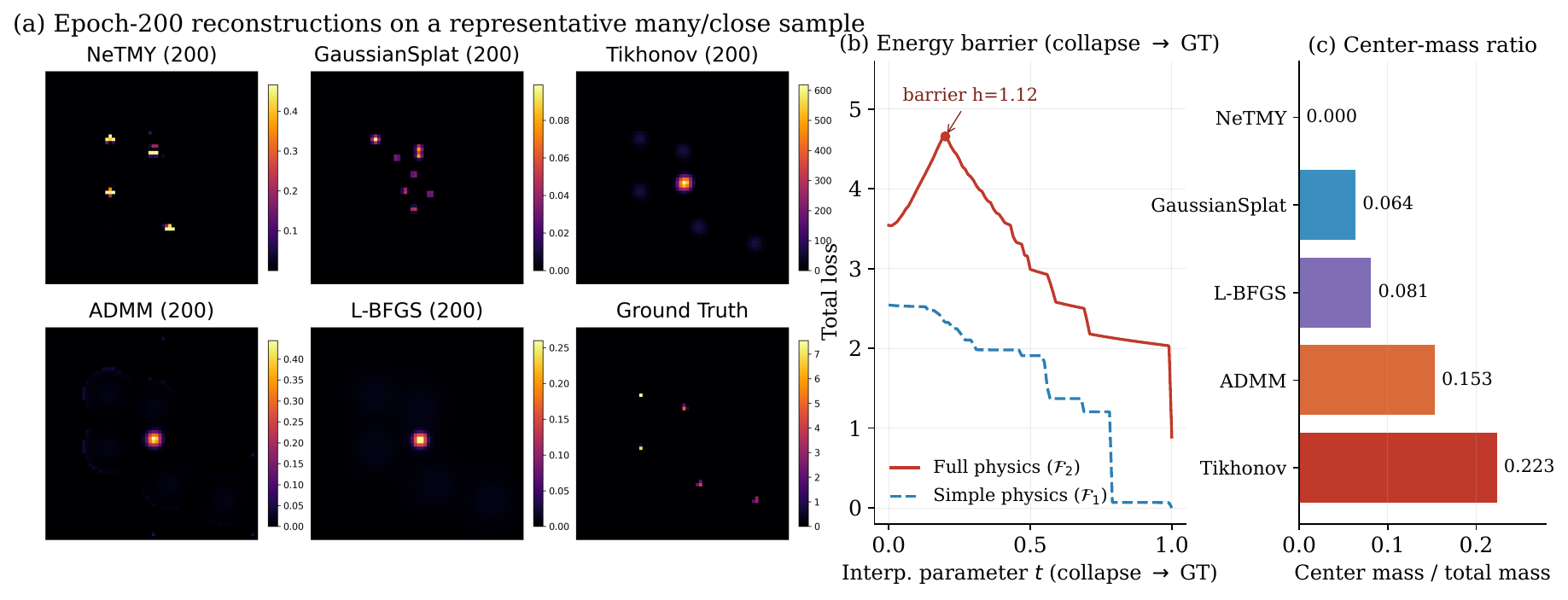}
    \caption{Optimization-geometry diagnostics. (a) Epoch-200 reconstructions. (b) Loss along $\rho(t)=(1-t)\rho_{\mathrm{collapse}}+t\,\rho_{\star}$. (c) Center-mass ratio at epoch~200.}
    \label{fig:mechanism}
    \vspace{-13pt}
\end{figure}

\textbf{Iter-0 gradient.} On a uniform initialization, the $\mathcal{F}_{2}$ data-fidelity gradient has a center-to-outer-ring magnitude ratio of $18.29\times$, with the peak at the grid center and not at any source: the (P2) signature predicted by Eq.~\eqref{eq:app-window}, executed verbatim by free-density solvers.

\textbf{Energy barrier.} The path from a centered-collapse iterate to the ground truth (Figure~\ref{fig:mechanism}b) crosses an $h\!\approx\!1.12$ barrier under $\mathcal{F}_{2}$ at $t\!=\!0.20$ but is monotonically decreasing under $\mathcal{F}_{1}$: a free-density solver in the centered minimum is hard to escape by gradient descent.

\textbf{Fixed budget.} At a $200$-iteration budget (Figure~\ref{fig:mechanism}c), center-mass ratios order as NeTMY ($0.00$) $<$ GaussianSplat ($0.064$) $<$ L-BFGS ($0.081$) $<$ ADMM ($0.153$) $<$ Tikhonov ($0.223$). For the parameterized solvers (NeTMY, GaussianSplat), the ranking tracks the effective smoothness of $G_{\theta}=J_{\theta}J_{\theta}^{\top}$. L-BFGS is a free-density solver and uses a quasi-Newton inverse-Hessian approximation, but still empirically suffers from the centered-collapse basin. The realized NeTMY iter-0 update has no singular center spike (App.~\ref{app:exp-mechanism}, Fig.~\ref{fig:app-netmy-delta}); a 32-run ADMM sweep confirms the operator effect (mean center-mass ratio $0.55\!\to\!0.78$ from $\mathcal{F}_{1}$ to $\mathcal{F}_{2}$; App.~\ref{app:exp-mechanism}).

\subsection{Component Ablation}
\label{sec:exp-ablation}

We ablate seven design choices on the full dataset (Table~\ref{tab:ablation}). Cumulatively removing \{annealing schedule, PE, multiscale, $\ell_{1}$, gate, $\mathcal{R}_{\mathrm{ds}}$\} degrades density MSE by $2$--$3\times$. The ordering tracks the ill-posedness mechanisms that motivated each component: annealed PE and PE address (P1), multiscale and gate address (P1) and (P3), and the direct-density loss provides the missing-amplitude proxy. TV has a small aggregate effect ($\times 1.1$) but qualitatively controls cross-shaped artifacts on dense scenes. Per-class and hyperparameter sweeps (LR, hidden width, PE octaves) are in Appendix~\ref{app:exp-ablation}.

\subsection{Real-Data Cross-Check on \texorpdfstring{$\alpha$-RuCl$_{3}$}{alpha-RuCl3}}
\label{sec:exp-realdata}

We test the operator-fidelity gap on the $\alpha$-RuCl$_{3}$ dataset of~\citet{kumar2024room} ($8$ NVs, SRIM depth prior $14\pm 5$\,nm), along two independent axes under each operator. We frame the result as a multi-axis consistency check, not a validation; the full four-step protocol (the two reported axes, an operator-neutral spectral calibration, and an additional power-law slope axis which is ambiguous on its own) is in Appendix~\ref{app:exp-realdata-full}.

\begin{wrapfigure}{r}{0.46\textwidth}
    \centering
    \vspace{-6pt}
    \includegraphics[width=0.46\textwidth]{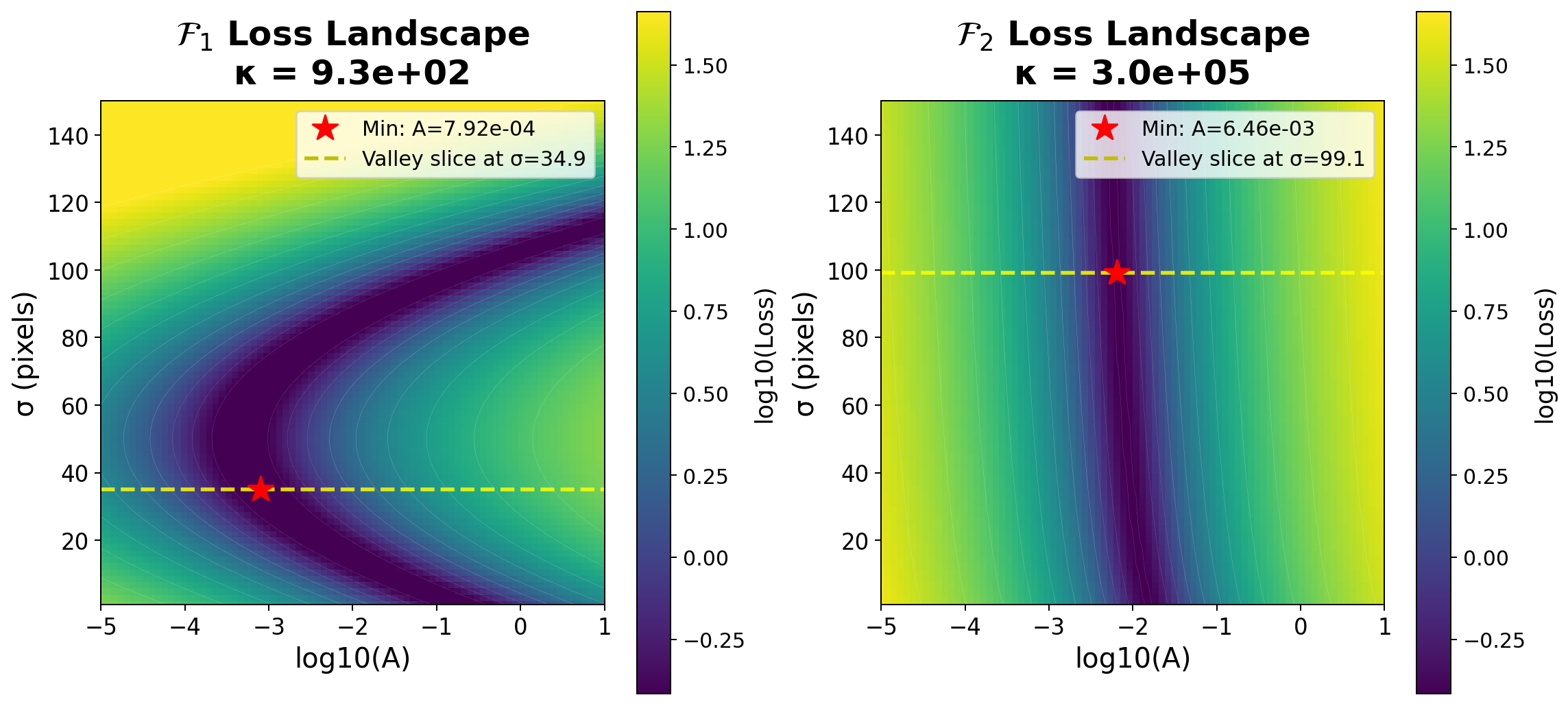}
    \caption{Loss landscape on $\alpha$-RuCl$_{3}$: $\mathcal{F}_{2}$ (left) is well-conditioned, $\mathcal{F}_{1}$ (right) is a degenerate valley.}
    \label{fig:realdata}
    \vspace{-15pt}
\end{wrapfigure}

\textbf{(i) Depth-amplitude.} Solving $\Gamma_{\mathrm{Ru}}^{\mathrm{pred}}(d)=\Gamma_{\mathrm{Ru}}^{\mathrm{meas}}$ for the implied NV depth $d^{\star}$ at fixed crystallographic density and moment puts $1/8$ NVs in $[9,20]$\,nm under $\mathcal{F}_{2}$ vs $0/8$ under $\mathcal{F}_{1}$, with point amplitude ratio $\Gamma^{\mathcal{F}_{1}}\!/\Gamma^{\mathcal{F}_{2}}\!\approx 210\times$ at $d=14.5$\,nm; the $\mathcal{F}_{1}$ inversion has no real solution for $7$ of the $8$ NVs. 

\textbf{(ii) Hessian conditioning.} On the Gaussian-density ansatz $\rho(\mathbf{r};A,\sigma_g)=A\exp(-\norm{\mathbf{r}}^{2}/2\sigma_g^{2})$, where $\sigma_g$ denotes the Gaussian width, over the $8$ NVs (Figure~\ref{fig:realdata}), the log-MSE condition numbers are $\kappa_{\mathcal{F}_{2}}=931$ and $\kappa_{\mathcal{F}_{1}}=301{,}139$ (ratio $\sim\!323\times$): the $\mathcal{F}_{2}$ landscape is a parabolic bowl while $\mathcal{F}_{1}$ is a degenerate valley along $A^{2}\sigma_g^{2}\!\approx\!\mathrm{const}$, a direct algebraic consequence of $\Gamma\propto A^{2}$ in $\mathcal{F}_{1}$.

Both reported axes (amplitude, conditioning) disfavor $\mathcal{F}_{1}$. The multi-axis evidence is consistent with \S~\ref{sec:exp-benchmark}--\ref{sec:exp-mechanism}. A complementary, non-NeTMY-specific consistency check that reproduces the dipolar $1/T_{1}\propto \sigma_b/D^{4}$ scaling under the forward operator is in App.~\ref{app:operator-further_variation}. A dense-to-sparse distribution-shift comparison and runtime are reported in App.~\ref{app:experiments}.

\section{Conclusion}
\label{sec:conclusion}

We framed NV-noise sensing inversion as a physics-faithful, ill-posed inverse problem in which forward-operator choice reshapes both the measurement distribution and the optimization geometry. Moving from a scalar/coherent approximation $\mathcal{F}_{1}$ to a tensor power-summed dipolar operator $\mathcal{F}_{2}$ exposes a center-collapse pathology in free-density solvers, traced to a centered iter-$0$ gradient and a max-normalization peak coupling. We proposed NeTMY, an amortization-free coordinate neural-field solver whose parameterization implements a positive semidefinite filter on the raw density-space gradient and, combined with annealed positional encoding, a multiscale curriculum, density gating, and energy-anchored scale correction, attains the best localization and distributional metrics on the cross-fidelity benchmark; an ablation isolates each component, and a multi-axis check on $\alpha$-RuCl$_{3}$~\citep{kumar2024room} disfavors $\mathcal{F}_{1}$ along amplitude and conditioning axes. More broadly, the operator-aware ranking, the gradient-filtering view, and the cross-fidelity protocol position NV relaxometry as a testbed for physics-faithful neural inverse problems.

\textbf{Limitations and Broader Impact.} NeTMY is per-measurement and trades wall-clock cost (roughly $100\times$ slower than classical baselines, App.~\ref{app:exp-runtime}) for label-free generalization. However, since the main goal of this method is to achieve a more precise reconstruction under the challenge of the lack of labeled data, this slowdown is not expected to dominate its use in research-grade reconstruction. Thus, it is acceptable for research-grade reconstruction; high-throughput deployment would benefit from amortization or warm-starting. The real-data check is a multi-axis consistency test on $8$ NVs under a Gaussian ansatz rather than a full validation, and we keep this scope explicit. The analysis is scoped to the fluctuation-dominated dipolar regime; coherent-coupling regimes and other quantum-sensing modalities are left to future work. NV noise sensing has numerous applications in materials science and biology, where better sparse reconstructions accelerate the study of spin-fluctuation phenomena and nanoscale magnetic textures, and we do not anticipate direct dual-use harms. The principal negative pathway is misplaced confidence in reconstructions outside the operator's modeled regime; we recommend reporting operator-fidelity caveats and consistency checks alongside downstream claims.



{\small
\bibliographystyle{plainnat}
\bibliography{main}
}


\appendix

\section{Forward Operator Derivation}
\label{app:operator-derivation}

This appendix expands the measurement model of Section~\ref{sec:meas-model} from first principles, makes the difference between the scalar $\mathcal{F}_{1}$ and tensor $\mathcal{F}_{2}$ operators explicit at the dipole level, and identifies the regimes in which each form is physically appropriate.

\subsection{Sample model and dipolar Green tensor}
\label{app:operator-setup}
We model the sample as a (possibly continuous) distribution of independent magnetic fluctuators with nonnegative density $\rho:\Omega\to\mathbb{R}_{\geq 0}$ over a 2D plane $\Omega\subset\mathbb{R}^{2}$. A widefield NV array sits at standoff height $z_{0}$, so the relative position from a sample point $\mathbf{r}_{\mathrm{src}}\in\Omega\times\{0\}$ to an NV pixel $\mathbf{r}\in\Omega\times\{z_{0}\}$ is $\mathbf{R}(\mathbf{r},\mathbf{r}_{\mathrm{src}})=(\mathbf{r}-\mathbf{r}_{\mathrm{src}},z_{0})$. The dipolar Green tensor
\begin{equation}
\label{eq:app-green}
G_{ia}(\mathbf{R}) \;=\; \frac{\mu_{0}}{4\pi}\,\frac{3R_{i}R_{a}-\norm{\mathbf{R}}^{2}\delta_{ia}}{\norm{\mathbf{R}}^{5}},\qquad i,a\in\{x,y,z\}
\end{equation}
maps a unit dipole moment along axis $a$ at $\mathbf{r}_{\mathrm{src}}$ to a magnetic-field component along axis $i$ at $\mathbf{r}$. For an NV with quantization axis $\mathbf{n}$, the projected response is $G_{\mathrm{nv},a}=\sum_{i}n_{i}G_{ia}$; for a $z$-aligned NV (the convention used throughout this paper) we have $G_{\mathrm{nv},a}=G_{za}=G_{az}$ by symmetry of $G$ in its two indices. Thermal magnetic fluctuations of a dipole at $\mathbf{r}_{\mathrm{src}}$ along axis $a$ produce a spectral density $\rho(\mathbf{r}_{\mathrm{src}})\,L(\omega;\omega_{L}(\mathbf{r}_{\mathrm{src}}))$ at the source, with $L$ the Lorentzian of Section~\ref{sec:meas-model}~\citep{tetienne2013spin,doherty2013nitrogen,casola2018probing}.

We retain three structural assumptions that are standard in the NV-relaxometry literature: (i) sample dipoles are independent thermal fluctuators (no inter-source coherence), (ii) the NV array reads out the noise spectrum at the readout pixel $\mathbf{r}$ rather than at the source pixel $\mathbf{r}_{\mathrm{src}}$ (a controlled approximation when $\omega_{L}$ varies slowly over the kernel support), and (iii) the source-spin fluctuation components $a\in\{x,y,z\}$ are independent thermal channels with equal variance, i.e., isotropic in the lab frame, so each channel contributes $\abs{G_{az}}^{2}$ to the NV-axis noise field at pixel $\mathbf{r}$. The physical regime in which these assumptions hold is discussed in Appendix~\ref{app:operator-regime}.

\subsection{Pointwise derivation of \texorpdfstring{$\mathcal{F}_{2}$}{F2} (incoherent thermal sum)}
\label{app:operator-incoherent}
For independent thermal fluctuators, discretize the source plane into cells of area $\Delta$ centered at $\{\mathbf{r}_{i}\}$ and write $\rho_{i}:=\rho(\mathbf{r}_{i})\,\Delta$ for the dimensionless source weight in cell $i$. The noise spectrum measured at NV position $\mathbf{r}$ is the sum of squared field amplitudes from all sources, summed over independent fluctuation channels~\citep{tetienne2013spin,van2015nanometre}:
\begin{align}
\mathcal{F}_{2}(\rho,\omega_{L})(\omega,\mathbf{r})
&\;=\;\sum_{a\in\{x,y,z\}}\sum_{i}\,\abs{G_{az}(\mathbf{R}_{i})}^{2}\,\rho_{i}\,L\!\bigl(\omega;\omega_{L}(\mathbf{r}_{i})\bigr) \nonumber\\
&\;\approx\;\biggl[\sum_{a\in\{x,y,z\}}\sum_{i}\abs{G_{az}(\mathbf{R}_{i})}^{2}\,\rho_{i}\biggr]\,L\!\bigl(\omega;\omega_{L}(\mathbf{r})\bigr),
\label{eq:app-s2}
\end{align}
where $\mathbf{R}_{i}=\mathbf{R}(\mathbf{r},\mathbf{r}_{i})$. The first line is the source-side exact form ($\mathcal{F}_{3}$ in our internal notation, used as the simulator $\mathcal{F}_{3}$ for ground-truth data generation in Section~\ref{sec:experiments}); the second line factors the Lorentzian out of the spatial sum, which is exact when $\omega_{L}$ is locally constant over the kernel support and equivalent (up to grid-area factors) to $[\sum_{a}(\abs{G_{az}}^{2}\!\ast\rho)(\mathbf{r})]\,L(\omega;\omega_{L}(\mathbf{r}))$ used by the FFT-based reconstruction stack. The contraction over $a$ is incoherent: distinct channels and distinct sources contribute independently, with no kernel-product cross-terms. Density enters Eq.~\eqref{eq:app-s2} \emph{linearly} in $\rho_{i}$, which is the property exploited by the scale-correction proof of Appendix~\ref{app:scale-correction}.
However, during the variation of the operator in the real-world dataset, the operator-implied internal-rate slopes are $b_{\mathrm{pred}}^{\mathcal{F}_{2}}=3.77$ and $b_{\mathrm{pred}}^{\mathcal{F}_{1}}=1.89$ via Kumar et al.~\citep{kumar2024room} $\alpha=1.06\pm 0.22$; the measured slope $b_{\mathrm{exp}}=1.43\pm 0.21$ is closer to $\mathcal{F}_{1}$. But slope alone does not penalize the amplitude error and is therefore ambiguous in isolation.

\subsection{Pointwise derivation of \texorpdfstring{$\mathcal{F}_{1}$}{F1} (scalar coherent square)}
\label{app:operator-coherent}
The simplified operator $\mathcal{F}_{1}$ retains only a single NV-axis projection and squares the convolved field as a final, post-superposition operation. Using the same discrete cell weights $\rho_i=\rho(\mathbf{r}_i)\Delta$, expanding the modulus square gives
\begin{align}
\mathcal{F}_{1}(\rho,\omega_{L})(\omega,\mathbf{r})
&\;=\;\biggl|\sum_{i}\,G_{\mathrm{nv}}(\mathbf{R}(\mathbf{r},\mathbf{r}_{i}))\,\rho_{i}\biggr|^{2}\,L(\omega;\omega_{L}(\mathbf{r}))\nonumber\\
&\;=\;\biggl[\underbrace{\sum_{i}\rho_{i}^{2}\,G_{\mathrm{nv}}(\mathbf{R}_{i})^{2}}_{\text{diagonal}}\;+\;\underbrace{\sum_{i\neq j}\rho_{i}\rho_{j}\,G_{\mathrm{nv}}(\mathbf{R}_{i})\,G_{\mathrm{nv}}(\mathbf{R}_{j})}_{\text{cross term}\,\mathcal{C}(\mathbf{r})}\biggr]\,L(\omega;\omega_{L}(\mathbf{r})),
\label{eq:app-s1}
\end{align}
where $\mathbf{R}_{i}=\mathbf{R}(\mathbf{r},\mathbf{r}_{i})$. The diagonal term reintroduces the squared cell weight $\rho_i^2$ instead of the linear cell weight $\rho_i$, so $\mathcal{F}_{1}(c\rho,\omega_{L})=c^{2}\,\mathcal{F}_{1}(\rho,\omega_{L})$; the cross-term $\mathcal{C}(\mathbf{r})$ couples pairs of distinct sources through the product of their NV-axis kernels, which has no counterpart in Eq.~\eqref{eq:app-s2}. When $\rho$ is sparse and well-separated relative to the support of $G_{\mathrm{nv}}$, $\mathcal{C}$ is small pointwise and $\mathcal{F}_{1}$ behaves as a (nonlinear-in-$\rho$) close cousin of $\mathcal{F}_{2}$; when sources are close, $\mathcal{C}$ dominates and the two operators differ in both magnitude and spatial structure.

\subsection{Direct simulator \texorpdfstring{$\mathcal{F}_{3}$}{F3} and the operator-fidelity check}
\label{app:operator-direct}
The two operators $\mathcal{F}_{1}$ and $\mathcal{F}_{2}$ are FFT-factorized: each evaluates a 2D convolution between a per-channel kernel and the source density, then multiplies by a Lorentzian factor at the readout pixel. The factorization is exact for $\mathcal{F}_{2}$ when the Larmor field $\omega_{L}$ is locally constant over the kernel support (Eq.~\eqref{eq:app-s2}), and is only an approximation when $\omega_{L}$ varies across the support. The direct simulator $\mathcal{F}_{3}$ removes this approximation by evaluating the source-side superposition pixel-by-pixel without FFT factorization,
\begin{align}
\label{eq:app-s3}
\mathcal{F}_{3}(\rho,\omega_{L})(\omega,\mathbf{r})
&\;=\;
\sum_{\mathbf{r}_{\mathrm{src}}\in\Omega}\rho(\mathbf{r}_{\mathrm{src}})\,P\!\bigl(\mathbf{R}(\mathbf{r},\mathbf{r}_{\mathrm{src}})\bigr)\,L\!\bigl(\omega;\omega_{L}(\mathbf{r}_{\mathrm{src}})\bigr),\\
P(\mathbf{R})
&\;=\sum_{a\in\{x,y,z\}}\abs{G_{az}(\mathbf{R})}^{2},\nonumber
\end{align}
where the Lorentzian is evaluated at each \emph{source} pixel $\mathbf{r}_{\mathrm{src}}$ rather than factored out at the readout pixel, restoring the per-source spectral response of the underlying physics. $\mathcal{F}_{3}$ is the source-side form (the first line of Eq.~\eqref{eq:app-s2}) without the locally-constant-$\omega_{L}$ factorization, and is implemented in float64 with a direct source-loop evaluator.

\textbf{Where \texorpdfstring{$\mathcal{F}_{3}$}{F3} is used.} The primary main-text benchmark in this paper is the cross-fidelity benchmark on a $512$-sample dataset (Table~\ref{tab:main}, Section~\ref{sec:exp-benchmark}, Appendix~\ref{app:exp-datasets}): spectra are produced by the source-side direct simulator $\mathcal{F}_{3}$ and inverted by $\mathcal{F}_{1}$ or $\mathcal{F}_{2}$, so the inversion operator is never reused as the data-generation operator and the inverse-crime risk~\citep{kaipio2007statistical} is avoided by construction. The matched-operator benchmarks ($\mathcal{F}_{1}/\mathcal{F}_{1}$ on $512$ samples and $\mathcal{F}_{2}/\mathcal{F}_{2}$ on $512$ samples; Table~\ref{tab:matched}, Appendix~\ref{app:exp-matched}) are reported alongside as a complementary view that includes amortized and untrained-prior baselines for which the $\mathcal{F}_{3}$ form is not implemented; they are intentionally inverse-crime by construction and we do not use them as the central evidence for the operator-fidelity claim.

\subsection{Operator difference}
\label{app:operator-diff}
Subtracting the discrete $\mathcal{F}_{2}$ approximation in Eq.~\eqref{eq:app-s2} from Eq.~\eqref{eq:app-s1} at fixed $\omega$ and $\mathbf{r}$,
\begin{equation}
\label{eq:app-opdiff}
\bigl(\mathcal{F}_{1}-\mathcal{F}_{2}\bigr)(\rho,\omega_{L})(\omega,\mathbf{r})
\;=\;\Bigl[\mathcal{C}(\mathbf{r})\;+\;\sum_{i}\bigl(\rho_{i}^{2}-\rho_{i}\bigr)G_{\mathrm{nv}}(\mathbf{R}_{i})^{2}\;-\;\sum_{a\neq z}\sum_{i}\rho_{i}\,\abs{G_{az}(\mathbf{R}_{i})}^{2}\Bigr]\,L(\omega;\omega_{L}(\mathbf{r})),
\end{equation}
i.e., the gap is the cross-term $\mathcal{C}$, plus the signed diagonal discrepancy $\rho_i^{2}-\rho_i$, minus the transverse-channel power $a\in\{x,y\}$ that is present in $\mathcal{F}_{2}$ and omitted by $\mathcal{F}_{1}$. These contributions are spatially structured across $\mathbf{r}$ and generally nonzero in multi-source scenes. Eq.~\eqref{eq:app-opdiff} is the algebraic origin of the inverse-landscape and ranking changes documented in Section~\ref{sec:experiments}.

\subsection{Physical regime of applicability}
\label{app:operator-regime}
Both forms are legitimate approximations in different physical regimes. $\mathcal{F}_{1}$ is appropriate for coherent-emitter superposition (e.g., a small number of phase-locked classical dipoles driven by the same coherent source), where the field amplitudes literally add before being squared by the detector. $\mathcal{F}_{2}$ is appropriate for incoherent thermal ensembles, where distinct fluctuators have independent random phases and only second moments of the field add~\citep{tetienne2013spin,van2015nanometre,casola2018probing}. NV relaxometry of nuclear, electronic, or magnonic spin baths is in the second regime~\citep{mzyk2022relaxometry,kumar2024room,midha2024optimized}, so $\mathcal{F}_{2}$ is the physically appropriate choice. We retain $\mathcal{F}_{1}$ as a benchmark only because it is computationally cheaper, has historically appeared in some open-source simulators~\citep{midha2024optimized}, and produces a meaningfully different inverse landscape that exposes the operator-fidelity sensitivity exploited in Section~\ref{sec:experiments}.

\subsection{Lorentzian and Larmor mapping}
\label{app:operator-lorentzian}
The spectral response in Eq.~\eqref{eq:app-s2} is the Lorentzian
\begin{equation}
\label{eq:app-lorentzian}
L(\omega;\omega_{L})\;=\;\frac{\gamma^{2}}{(\omega-\omega_{L})^{2}+\gamma^{2}},
\end{equation}
with linewidth $\gamma$ fixed by the experimental setup (we use $\gamma=0.5$\,GHz throughout, matching the simulator). The Larmor map $\omega_{L}:\Omega\to\mathbb{R}_{+}$ encodes the per-pixel quasi-static detuning of the local resonance and is treated as a second unknown field in the inverse problem. In the FFT-based factorization used by all reconstruction methods, $L(\omega;\omega_{L}(\mathbf{r}))$ is evaluated at the readout pixel $\mathbf{r}$ rather than at the source pixel $\mathbf{r}_{\mathrm{src}}$, which is exact when $\omega_{L}$ is approximately constant over the kernel support and a controlled approximation otherwise. The direct simulator $\mathcal{F}_{3}$ used to generate ground-truth data evaluates $L$ at $\mathbf{r}_{\mathrm{src}}$ and is therefore closer to the underlying physics; in the main-text cross-fidelity benchmark we use $\mathcal{F}_{3}$ for data generation and either $\mathcal{F}_{1}$ or $\mathcal{F}_{2}$ for reconstruction (Appendix~\ref{app:operator-direct}).

\subsection{Robustness to spatially varying Larmor frequencies}
\label{app:operator-lorentzian_robustness}
We further validated the robustness of the FFT-factorized solver $F_2$ against spatial variation of the Larmor frequency within one dipolar kernel footprint. In the real-data setting, the Larmor frequency is primarily set by the externally applied magnetic field and is therefore nearly uniform across the reconstruction field of view, with only weak local perturbations. The relevant control parameter is the dimensionless ratio
$\chi \sim \Delta \omega_L^{(\mathrm{kernel})}/\gamma$,
where $\Delta \omega_L^{(\mathrm{kernel})}$ denotes the in-kernel variation of $\omega_L$ and $\gamma$ is the effective linewidth. In a controlled stress test, the relative forward error of $F_2$ with respect to the direct source-side operator $F_3$ remained small when $\chi \ll 1$ (about $2.6\times 10^{-2}$ at $\chi \approx 0.3$), but increased substantially as $\chi$ approached unity (about $1.6\times 10^{-1}$ at $\chi \approx 1.08$). These results confirm that $F_2$ is reliable in the broad-linewidth, slowly varying regime relevant to our experiments.

\subsection{Further physical variation of the forward operator}
\label{app:operator-further_variation}
Figure~\ref{app:PV-1} shows the spectrum noise generated by different components of the green kernel. The shape and scale of the spectrum noise stay consistent with the physical formula
\begin{figure}[H]
\centering
\includegraphics[width=0.8\linewidth]{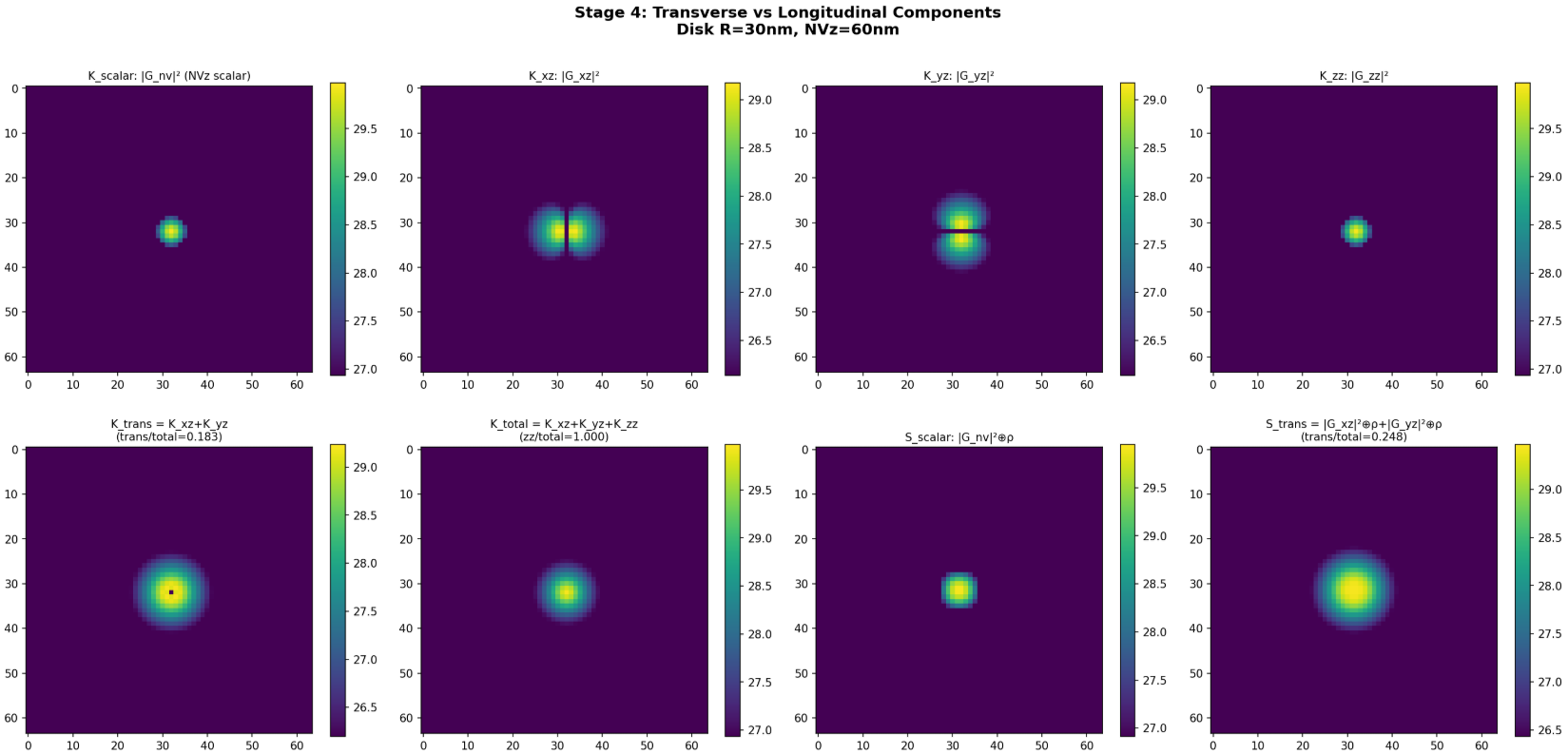}
\caption{Verification of the Green kernel in physical solver}
\label{app:PV-1}
\end{figure}
According to~\citep{tetienne2013spin}, the experiment detects the longitudinal relaxation time $T_1$ of the dipole spin bath (with bath surface density $\sigma_b$) on the nanodiamond with the NV center.

The cited work gives the formula
\begin{equation}
    \frac{1}{T_1}
    =
    \frac{1}{T_1^{\text{bulk}}}
    +
    \left(
    \frac{48 \mu_0^2 \gamma_e^4 \hbar^2 C_S}{\pi D^4}
    \right)
    \left(
    \frac{\sigma_b R(\sigma_b)}{\omega_0^2 + R(\sigma_b)^2}
    \right)
\end{equation}
and 
\begin{equation}
    B_\perp^2
    =
    \left(
    \frac{4 \mu_0 \gamma_e \hbar}{\pi}
    \right)^2
    \pi C_S \frac{\sigma_b}{D^4}
\end{equation}
where $C_S = \frac{1}{2S+1}\sum_{m=-S}^{S} m^2
    = \frac{S(S+1)}{3}$ and $R(\sigma_b) = \frac{1}{\tau_c}$. The following relationship has been clearly stated in the research paper
\begin{equation}
    B_\perp^2 \propto \frac{\sigma_b}{D^4}
    \label{eq:Spin_PV1}
\end{equation}
\begin{equation}
    \frac{1}{T_1} \propto \frac{\sigma_b}{D^4}
    \label{eq:Spin_PV2}
\end{equation}
To verify the physical correctness of the forward operator in the research, the experiment is reproduced. As shown in Figure~\ref{app:PV-2}, the noise power is inversely proportional to the fourth power of the diameter of the diamond $D$.
\begin{figure}[H]
    \centering
    \begin{subfigure}{0.28\textwidth}
        \centering
        \includegraphics[width=\linewidth]{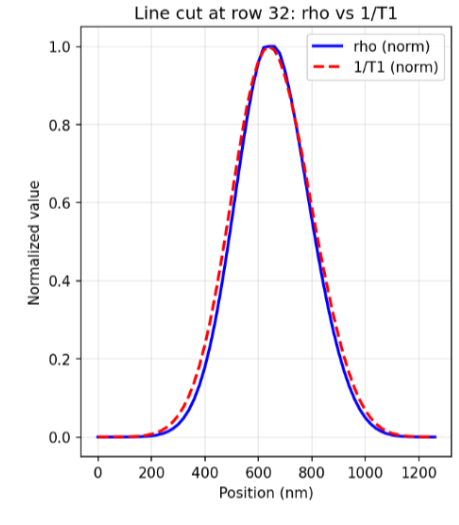}
        \caption{Normalized line cut at row 32 comparing $\rho$ and 
        $1/T_1$ spatial profiles, demonstrating the qualitative 
        agreement between the two distributions.}
        \label{fig:left}
    \end{subfigure}
    \hfill
    \begin{subfigure}{0.7\textwidth}
        \centering
        \includegraphics[width=\linewidth]{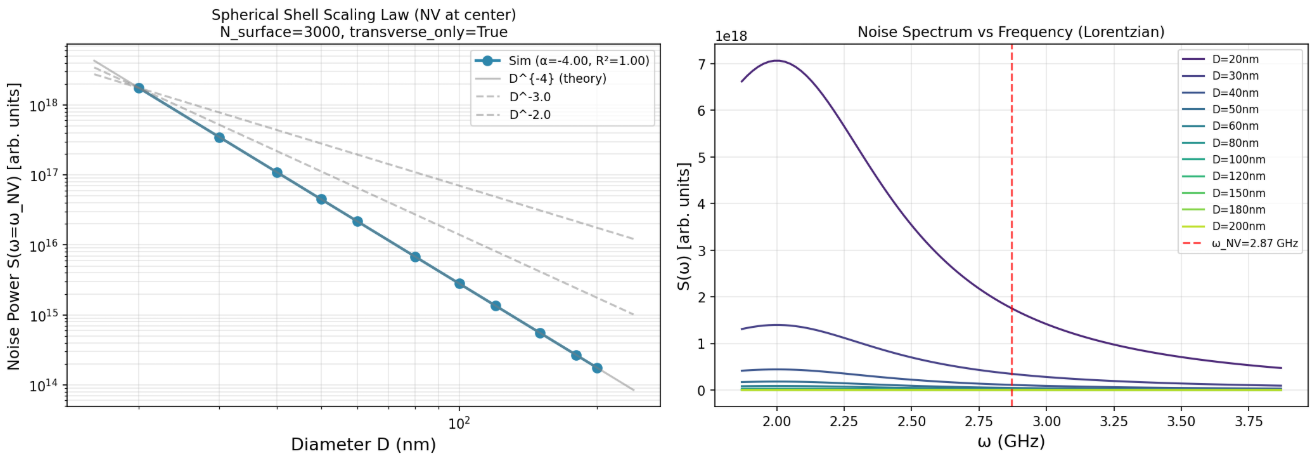}
        \caption{Forward operator reproduction of the Tetienne 
    \textit{et al.}\ dipole bath model~\cite{tetienne2013spin}. 
    \textbf{Left:} Simulated noise power (dots) follows the theoretical 
    $D^{-4}$ scaling, and the Lorentzian spectrum 
    $S(\omega)$ shows stronger noise for smaller $D$ near 
    $\omega_\text{NV}=2.87\,\text{GHz}$.
    \textbf{Right:} The reconstructed $1/T_1(x,y)$ map matches the input 
    $\rho(x,y)$ distribution (Pearson $r=99.55\%$), confirming 
    $1/T_1(x) \approx (|G|^2 * \rho)(x)$ }
        \label{fig:right}
    \end{subfigure}
    \caption{The reproduction of the relationship between magnetic noise and $\frac{1}{T_1}$ using the forward operator and the relationship between magnetic noise and diameter using the forward operator}
    \label{app:PV-2}
\end{figure}
From Eq.~\eqref{eq:Spin_PV1} and Eq.~\eqref{eq:Spin_PV2}, we can find out that the magnetic noise has the same order of magnitude, which means

\begin{equation}
    \frac{1}{T_1(x)} \approx S = \left( |G|^2 * \rho \right)(x)
\end{equation}
Thus for a stated dipole bath, the distribution of $\rho$ and $\frac{1}{T_1(x)}$ should be similar, which is also followed by the forward operator.

\section{Data Fidelity, Scale Ambiguity, and Energy-Ratio Correction}
\label{app:scale-correction}

This appendix gives the full form of the data-fidelity term $\mathcal{D}$ used in Eq.~\eqref{eq:invprob}, records the formal proof that the pre-normalization scale of $\rho$ is unidentifiable from $\mathcal{D}$ alone, and shows how an energy-ratio anchor restores identifiability. The proof is operator-agnostic and applies to any reconstruction method that minimizes Eq.~\eqref{eq:invprob}; the specific scale-correction \emph{post-processing step} adopted by NeTMY is described in Section~\ref{sec:method} and Appendix~\ref{app:method-scale}.

\subsection{Pixelwise log-MSE form of \texorpdfstring{$\mathcal{D}$}{D}}
\label{app:fidelity}
With $N(\mathbf{r};S)=\sum_{\omega\in W}S(\omega,\mathbf{r})$, $\widehat{N}(\mathbf{r};S)=N(\mathbf{r};S)/\max_{\mathbf{r}'}N(\mathbf{r}';S)$, and $\ell(\mathbf{r};S)=\log_{10}(\widehat{N}(\mathbf{r};S)+10^{-10})$, the data-fidelity discrepancy in Eq.~\eqref{eq:invprob} is
\begin{equation}
\label{eq:fidelity}
\mathcal{D}\bigl(\mathcal{F}(\rho,\omega_{L}),S_{\mathrm{obs}}\bigr)
\;=\;
\frac{1}{|\Omega|}\sum_{\mathbf{r}\in\Omega}\Bigl(\ell\!\bigl(\mathbf{r};\mathcal{F}(\rho,\omega_{L})\bigr)\;-\;\ell\bigl(\mathbf{r};S_{\mathrm{obs}}\bigr)\Bigr)^{2},
\end{equation}
where $|\Omega|$ denotes the number of grid pixels in $\Omega$. The two preprocessing steps (frequency summation followed by spatial max-normalization) are standard in NV relaxometry to suppress per-sample amplitude variation and emphasize spatial-distribution mismatch~\citep{kumar2024room,midha2024optimized,casola2018probing}; the small additive constant $10^{-10}$ avoids numerical singularity on background pixels.

\subsection{Setup: max-normalization induces a scale-flat ray}
\label{app:scale-setup}
Let $S_{\mathrm{obs}}$ be the noise spectrum produced by an unknown ground-truth pair $(\rho_{\star},\omega_{L,\star})$ under the linear operator $\mathcal{F}_{2}$, and let $\widehat{N}(\,\cdot\,;S)$ denote its frequency-summed, max-normalized noise map (Section~\ref{sec:inverse-problem}). Linearity gives $\mathcal{F}_{2}(c\rho,\omega_{L})=c\,\mathcal{F}_{2}(\rho,\omega_{L})$ for any $c>0$, so
\begin{equation}
\label{eq:app-scaleflat}
\widehat{N}(\,\cdot\,;\mathcal{F}_{2}(c\rho,\omega_{L}))\;=\;\widehat{N}(\,\cdot\,;\mathcal{F}_{2}(\rho,\omega_{L}))\qquad\forall\,c>0.
\end{equation}
The fidelity term $\mathcal{D}$ in Eq.~\eqref{eq:fidelity} is therefore exactly invariant along the ray $\{c\rho:c>0\}$, and the data alone cannot identify the absolute density scale. This is a finite-dimensional instance of the bilinear-identifiability symptom analyzed in compressed-sensing and blind-deconvolution literature~\citep{ahmed2013blind,ling2015self,lee2016blind}.

\subsection{Energy-ratio correction is exact under \texorpdfstring{$\mathcal{F}_{2}$}{F2} in the noiseless limit}
\label{app:scale-proof}
\begin{proposition}[Scale correction under $\mathcal{F}_{2}$]
\label{prop:scale}
Let $\mathcal{F}$ be linear in $\rho$ and assume the noiseless observation $S_{\mathrm{obs}}=\mathcal{F}(\rho_{\star},\omega_{L,\star})$. Let $\widehat{\rho}$ be a candidate reconstruction with the same shape as $\rho_{\star}$ up to scale, $\widehat{\rho}=c^{-1}\rho_{\star}$ for some unknown $c>0$. Define
\[
E_{\mathrm{obs}}\;=\;\sum_{\omega,\mathbf{r}}S_{\mathrm{obs}}(\omega,\mathbf{r}),\qquad
E_{\mathrm{pred}}\;=\;\sum_{\omega,\mathbf{r}}\mathcal{F}(\widehat{\rho},\omega_{L,\star})(\omega,\mathbf{r}),
\]
and $\alpha=E_{\mathrm{obs}}/E_{\mathrm{pred}}$. Then $\alpha\widehat{\rho}=\rho_{\star}$ exactly, and $\alpha$ depends only on the observation and the reconstruction.
\end{proposition}
\begin{proof}
Linearity and $\widehat{\rho}=c^{-1}\rho_{\star}$ give $\mathcal{F}(\widehat{\rho},\omega_{L,\star})=c^{-1}\mathcal{F}(\rho_{\star},\omega_{L,\star})$ pointwise. Hence $E_{\mathrm{pred}}=c^{-1}E_{\mathrm{obs}}$, so $\alpha=E_{\mathrm{obs}}/E_{\mathrm{pred}}=c$ and $\alpha\widehat{\rho}=\rho_{\star}$.
\end{proof}
With additive sensor noise $S_{\mathrm{obs}}=\mathcal{F}(\rho_{\star},\omega_{L,\star})+\varepsilon$, write $E_{\varepsilon}=\sum_{\omega,\mathbf{r}}\varepsilon(\omega,\mathbf{r})$ and $E_{\star}=\sum_{\omega,\mathbf{r}}\mathcal{F}(\rho_{\star},\omega_{L,\star})$. Then $\alpha=c\,(1+E_{\varepsilon}/E_{\star})$, so no deterministic exactness is claimed; if $\mathbb{E}\varepsilon=0$, the correction is unbiased in expectation and its relative fluctuation is $O(\sigma_{\varepsilon}\sqrt{|W|\,|\Omega|}/E_{\star})$. A scale-correction regularizer that penalizes deviation of $\alpha$ from $1$ (or, equivalently, anchors $\rho$ to the observed energy after every step) restores identifiability without any oracle quantity. Two further caveats are worth recording independently of method. (i) The proof requires linearity of $\mathcal{F}$ in $\rho$, which holds for $\mathcal{F}_{2}$ but not $\mathcal{F}_{1}$ (handled below). (ii) The shape assumption $\widehat{\rho}=c^{-1}\rho_{\star}$ is exact only when the reconstruction recovers the true support; in practice we observe graceful degradation and a residual scale bias proportional to the shape error.

\subsection{Square-root correction under \texorpdfstring{$\mathcal{F}_{1}$}{F1}}
\label{app:scale-s1}
For the quadratic operator $\mathcal{F}_{1}(c\rho,\omega_{L})=c^{2}\mathcal{F}_{1}(\rho,\omega_{L})$, the same noiseless calculation with $\widehat{\rho}=c^{-1}\rho_{\star}$ gives $E_{\mathrm{pred}}=c^{-2}E_{\mathrm{obs}}$, so the physically matched closed form is
\begin{equation}
\label{eq:app-c1}
c_{1}\;=\;\sqrt{\,E_{\mathrm{obs}}/E_{\mathrm{pred}}\,},\qquad c_{2}\;=\;E_{\mathrm{obs}}/E_{\mathrm{pred}},
\end{equation}
i.e., the $\mathcal{F}_{1}$-form correction is the square root of the $\mathcal{F}_{2}$-form correction in the noiseless limit. The two forms agree only at the trivial fixed point $E_{\mathrm{obs}}=E_{\mathrm{pred}}$. We use the operator-matched convention in the main-body tables:
$c_1$ for $\mathcal F_1$ and $c_2$ for $\mathcal F_2$; numerically, the asymmetry is small for the localization-oriented metrics (GMSD, Hungarian F1, Sliced-WD) reported in Section~\ref{sec:experiments} because none of them depend on the absolute density scale, but it does affect density MSE.

\section{Ill-posedness Diagnostics}
\label{app:illposed-diagnostics}

This appendix gives the formal counterparts of the four ill-posedness statements (P1)--(P4) in Section~\ref{sec:illposed}, and quantifies the resulting forward-Jacobian and gradient structure that motivates the Method of Section~\ref{sec:method}.

\subsection{(P1) Exponential frequency suppression}
\label{app:illposed-frequency}
\begin{lemma}[Frequency decay of the tensor channel power]
\label{lem:freqdecay}
For $z_{0}>0$ and $a\in\{x,y,z\}$, the 2D Fourier transform of the source-plane Green-tensor channel evaluated at standoff $z_{0}$ satisfies $|\widehat{G_{az}}(\mathbf{k};z_{0})|=O(k\,e^{-k z_{0}})$ for $k=\norm{\mathbf{k}}$. Consequently, for fixed $z_{0}$ there exist constants $C_{a,z_{0}}$ and $m<\infty$ such that
\begin{equation}
\label{eq:app-decay}
\bigl|\widehat{\,\abs{G_{az}}^{2}\,}(\mathbf{k};z_{0})\bigr|
\;\leq\; C_{a,z_{0}}\,(1+k)^{m}\,e^{-k z_{0}},
\end{equation}
so the linear forward map $\rho\mapsto\sum_{a}|G_{az}|^{2}\!\ast\rho$ suppresses spatial mode $\mathbf{k}$ with exponential envelope $e^{-k z_{0}}$ up to algebraic factors in $k$.
\end{lemma}
\begin{proof}
Let $\Phi(\mathbf{r},z)=(4\pi)^{-1}(\norm{\mathbf{r}}^{2}+z^{2})^{-1/2}$ be the half-space magnetostatic scalar potential. Its 2D Fourier transform in $\mathbf{r}$ is, up to convention constants, $\widehat{\Phi}(\mathbf{k},z)=k^{-1}e^{-k z}$. Each component $G_{ia}$ is a second derivative of $\Phi$ in real space; Fourier differentiation contributes factors $i k_{i}$ or $-k$, hence $|\widehat{G_{az}}(\mathbf{k};z_{0})|\leq C_{a}\,k\,e^{-k z_{0}}$ for a constant $C_{a}$. Since $|G_{az}|^{2}$ is the pointwise square of a real-valued function, its 2D Fourier transform is the autoconvolution $\widehat{|G_{az}|^{2}}=(2\pi)^{-2}\widehat{G_{az}}\!\ast\!\widehat{\overline{G_{az}}}$. Therefore
\[
\bigl|\widehat{|G_{az}|^{2}}(\mathbf{k})\bigr|
\;\leq\; C\!\int_{\mathbb{R}^{2}}\!\norm{\mathbf{q}}\,\norm{\mathbf{k}-\mathbf{q}}\,e^{-z_{0}(\norm{\mathbf{q}}+\norm{\mathbf{k}-\mathbf{q}})}\,d\mathbf{q}
\;\leq\; C_{z_{0}}\,(1+k)^{m}\,e^{-k z_{0}},
\]
where the second inequality uses the triangle bound $\norm{\mathbf{q}}+\norm{\mathbf{k}-\mathbf{q}}\geq k$ and bounds the residual integral by an algebraic factor in $k$ for fixed $z_{0}$.
\end{proof}
The implication for optimization is that the eigenvalues of the linearized forward map drop off geometrically across spatial scales, making the inverse map exponentially ill-conditioned at high $k$. We exploit this prediction in two concrete ways: (a) we use a coordinate neural field with annealed positional encoding so that low-$k$ modes are fit before high-$k$ modes (Section~\ref{sec:method}), and (b) we predict (and measure) a measurable degradation when the curriculum order is reversed.

\subsection{(P2) Finite-window center bias}
\label{app:illposed-window}
On a bounded sensing window $\Omega$, the convolutional footprint of a single source pixel at $\mathbf{r}_{0}$ under $\mathcal{F}_{2}$ is the translated channel-summed power kernel $\sum_{a}\abs{G_{az}(\,\cdot\,-\mathbf{r}_{0})}^{2}$, whose energy on $\Omega$ is maximized when $\mathbf{r}_{0}$ is at the window center and decreases monotonically as $\mathbf{r}_{0}$ approaches the boundary, because the translated footprint is increasingly truncated by $\partial\Omega$. The forward Jacobian column for source pixel $\mathbf{r}_{0}$ has squared norm
\begin{equation}
\label{eq:app-window}
\norm{\frac{\partial\mathcal{F}_{2}}{\partial\rho(\mathbf{r}_{0})}}^{2}
\;=\;\int_{W}\!\int_{\Omega}\biggl(\sum_{a\in\{x,y,z\}}\abs{G_{az}(\mathbf{r}-\mathbf{r}_{0})}^{2}\biggr)^{2}\,L(\omega;\omega_{L}(\mathbf{r}))^{2}\,d\mathbf{r}\,d\omega,
\end{equation}
which expands into auto-channel terms $\sum_{a}\abs{G_{az}}^{4}$ plus positive cross-channel terms $\sum_{a\neq a'}\abs{G_{az}}^{2}\abs{G_{a'z}}^{2}$. Since the integrand is a nonnegative localized footprint, truncation by $\partial\Omega$ preserves the center-vs-edge asymmetry. Even from a uniform initialization $\rho_{0}\equiv c$, the iter-$0$ gradient of the data-fidelity term therefore points more strongly toward central pixels than toward boundary pixels, before any structure has formed. We measure this directly in Section~\ref{sec:experiments}: the iter-$0$ physics-loss gradient under $\mathcal{F}_{2}$ has a center-to-outer ratio of $18.29\times$ on a representative sample, with the peak at the grid center, even though the ground-truth density is off-center. (P2) is therefore a bias of the operator and the window, not a bug of any one solver.

\subsection{(P3) Max-normalization peak coupling}
\label{app:illposed-maxnorm}
The data-fidelity term in Eq.~\eqref{eq:fidelity} compares max-normalized noise maps $\widehat{N}_{i}=N_{i}/M$, $M=\max_{j}N_{j}$. Differentiating $\widehat{N}$ with respect to the unnormalized $N$:
\begin{equation}
\label{eq:app-maxnormgrad}
\frac{\partial\widehat{N}_{i}}{\partial N_{j}}
\;=\;\frac{\delta_{ij}}{M}\;-\;\frac{N_{i}}{M^{2}}\,\frac{\partial M}{\partial N_{j}},\qquad
\frac{\partial M}{\partial N_{j}}\;=\;\delta_{jk}\quad\text{where}\quad k=\arg\max_{j}N_{j}.
\end{equation}
For a non-peak pixel ($j\neq k$) this reduces to $\delta_{ij}/M$ (purely local), but for the peak pixel ($j=k$) Eq.~\eqref{eq:app-maxnormgrad} contains an extra term $-N_{i}/M^{2}$ that couples the gradient at $N_{k}$ to \emph{every} pixel $i$ where $N_{i}$ is appreciable, with weight proportional to $N_{i}$. The chain-rule combination of Eq.~\eqref{eq:app-maxnormgrad} with the (P2) center bias amplifies the peak whenever the current peak is at the window center, which is exactly the iter-$0$ situation produced by (P2) on a uniform initialization. The qualitative consequence is that free-density methods (Tikhonov, ADMM) following $\nabla_{\rho}\mathcal{D}$ literally execute this center-amplifying step at iter $0$, and once a centered peak forms it self-reinforces through Eq.~\eqref{eq:app-maxnormgrad}. We document the resulting energy-barrier between the centered-collapse minimum and the ground-truth density in Section~\ref{sec:experiments}.

\subsection{(P4) Source merging and joint \texorpdfstring{$(\rho,\omega_{L})$}{(rho, omega\_L)} ambiguity}
\label{app:illposed-resolution}
The effective point-spread width $w_{\mathrm{psf}}$ of $\mathcal{F}_{2}$ is set by the standoff $z_{0}$ through the kernel decay of Lemma~\ref{lem:freqdecay}: each tensor channel $|G_{az}|^{2}$ has spatial extent $\sim z_{0}$ on the source side, so two sources separated by $\Delta r<w_{\mathrm{psf}}\sim z_{0}$ produce nearly identical forward responses and cannot be disambiguated from the data alone. A simple lower-bound estimate~\citep{grech2008review} useful for our setup is
\begin{equation}
\label{eq:app-psf}
w_{\mathrm{psf}}\;\approx\;2 z_{0}\sqrt{\,\log 2\,/\,(k z_{0}+3)\,}\quad\text{at characteristic mode}\;k,
\end{equation}
which we use as a working threshold to label "ill-conditioned" data regimes in Section~\ref{sec:experiments}. Beyond pairwise merging, in many-source scenes a continuous family of densities can fit the same forward measurement to within noise, producing axis-aligned cross-shaped artifacts when the optimizer settles on a low-cost, anisotropic explanation~\citep{zhang2018structured}.

A separate non-uniqueness arises from the joint role of $\rho$ and $\omega_{L}$ in Eq.~\eqref{eq:s1s2}: at any pixel $\mathbf{r}$ where the source contribution $\sum_{a}|G_{az}|^{2}\!\ast\rho(\mathbf{r})$ is small, the Lorentzian $L(\omega;\omega_{L}(\mathbf{r}))$ multiplies a near-zero coefficient and the fidelity term places no constraint on $\omega_{L}(\mathbf{r})$. The Larmor map is therefore identifiable only on the support of $\rho$. NeTMY enforces this scope at the architectural level (Appendix~\ref{app:method-arch}, Eq.~\eqref{eq:app-larmor-mask}): the predicted Larmor field is multiplicatively masked by an indicator on the predicted support of $\rho_{\theta}$, so gradients flow into $\omega_{L}(\mathbf{r})$ only where $\rho_{\theta}(\mathbf{r})$ is appreciable. Without such gating, any local minimum acquires arbitrary Larmor values on the background, contaminating the per-pixel loss landscape and slowing convergence on the support.

\subsection{Summary: which pathology drives which experiment}
\label{app:illposed-summary}
Table~\ref{tab:app-pathologies} maps each ill-posedness pathology to the empirical signature it predicts and to the experimental section that reports the measurement.
\begin{table}[ht]
\centering
\small
\setlength{\tabcolsep}{4pt}
\begin{tabular}{p{0.16\textwidth}p{0.30\textwidth}p{0.22\textwidth}p{0.22\textwidth}}
\toprule
Pathology & Predicted signature & Reported in & Mitigated by \\
\midrule
(P1) Frequency suppression & Coarse-to-fine $>$ fine-to-coarse curriculum & Sec.~\ref{sec:experiments}, ablation & Sec.~\ref{sec:method}, annealed PE \\
(P2) Window center bias & Iter-$0$ centered gradient & Sec.~\ref{sec:experiments}, mechanism & Sec.~\ref{sec:method}, parameterization \\
(P3) Max-norm peak coupling & Energy barrier at center collapse & Sec.~\ref{sec:experiments}, mechanism & Sec.~\ref{sec:method}, gating + PE \\
(P4) PSF merging, Larmor ambiguity & Cross artifacts, background Larmor noise & Sec.~\ref{sec:experiments}, failure & Sec.~\ref{sec:method}, masked Larmor head \\
\bottomrule
\end{tabular}
\caption{Mapping from ill-posedness pathology to empirical signature, the experimental section that reports it, and the design choice in Section~\ref{sec:method} that mitigates it.}
\label{tab:app-pathologies}
\end{table}

\section{Method Details}
\label{app:method}

This appendix gives the full functional and algorithmic details deferred from Section~\ref{sec:method}: the architecture of the coordinate MLP and density/Larmor heads, the frequency-annealing schedule, the per-stage training schedule, the explicit forms of all loss terms, the post-training scale correction, and the formal derivation of the Jacobian-induced filtering relation $\Delta\rho \approx -\eta\,G_{\theta}\nabla_{\rho}\mathcal{L}$.

\subsection{Architecture and Output Heads}
\label{app:method-arch}
The coordinate field $f_{\theta}$ in Eq.~\eqref{eq:method-field} is realized by the network $f_{\theta}(\mathbf{x})=\mathrm{Heads}\bigl(\mathrm{MLP}_{\theta}(\gamma_{\beta}(\mathbf{x}))\bigr)$, where $\mathbf{x}=(x,y)\in\Omega$, $\gamma_{\beta}$ is the annealed Fourier encoder (Appendix~\ref{app:method-anneal}), $\mathrm{MLP}_{\theta}$ is a 6-layer fully connected backbone, and $\mathrm{Heads}$ is the two-output transform of Eq.~\eqref{eq:method-density-head} and Eq.~\eqref{eq:app-larmor-mask}. Coordinates are normalized to the cube $[-1,1]^{2}$ before encoding so that the Fourier basis frequencies are device-agnostic. Table~\ref{tab:app-method-arch} summarizes the architectural choices used in every reported run.

\begin{table}[h]
\centering
\small
\setlength{\tabcolsep}{6pt}
\begin{tabular}{ll}
\toprule
Component & Setting \\
\midrule
Coordinate input & $(x,y)\in[-1,1]^{2}$, normalized by half grid extent \\
Positional encoding & annealed Fourier features, $K=12$ octaves, base 2 \\
Encoded input dim & $2 + 2\cdot 2\cdot K = 50$ \\
Backbone & 6 fully connected layers, hidden $320$, $\tanh$ activation \\
Skip connection & encoded input re-injected at layer 3 \\
Output channels & $3$: density logit $h_{\rho}$, density gate $g$, Larmor logit $h_{\omega}$ \\
Density head & $\rho_{\theta}=\mathrm{softplus}(h_{\rho})\,\sigma(g)$ \\
Larmor head & $\omega_{L,\theta}=\omega_{L,\min}+(\omega_{L,\max}-\omega_{L,\min})\,\sigma(h_{\omega})$, with mask in Eq.~\eqref{eq:app-larmor-mask} \\
Parameter count & $\approx\!4.5\times 10^{5}$ \\
Optimizer & AdamW, weight decay $10^{-4}$, gradient clip $\norm{\nabla\theta}\leq 1$ \\
Learning-rate schedule & cosine annealing per stage to $0.01\cdot\eta_{\mathrm{base}}$ \\
\bottomrule
\end{tabular}
\caption{NeTMY architectural and optimization settings. All reported main-body and ablation runs use these defaults except where the ablation explicitly toggles a component.}
\label{tab:app-method-arch}
\end{table}

The density-head transform of Eq.~\eqref{eq:method-density-head} is a gated softplus: the softplus enforces nonnegativity, and the sigmoid gate $\sigma(g)\in(0,1)$ acts as a pixelwise on/off switch that the optimizer can drive close to zero without saturating the softplus. Empirically, removing the gate (i.e., $\rho_{\theta}=\mathrm{softplus}(h_{\rho})$) makes background suppression substantially harder and weakens sparse-localization metrics (Section~\ref{sec:experiments}, ablation).

The Larmor head is sigmoid-mapped into the device-determined band and then multiplied by a hard predicted-support mask,
\begin{equation}
\label{eq:app-larmor-mask}
\omega^{\mathrm{mask}}_{L,\theta}(\mathbf{r})
\;=\;m_{\theta}(\mathbf{r})\,\omega^{\mathrm{band}}_{L,\theta}(\mathbf{r}),\qquad
m_{\theta}(\mathbf{r})\;=\;\mathbf{1}\bigl\{\rho_{\theta}(\mathbf{r}) > \tau\,\max_{\mathbf{r}'}\rho_{\theta}(\mathbf{r}')\bigr\},\qquad\tau=0.3,
\end{equation}
with $m_{\theta}$ treated as a stop-gradient constant during back-prop. The nondifferentiable indicator is therefore not differentiated, and the Larmor logit receives gradients only where $m_{\theta}=1$. Pixels with $m_{\theta}=0$ are assigned the fixed value $0$, which lies outside $[\omega_{L,\min},\omega_{L,\max}]$ but is not interpreted as a physical Larmor frequency: since $\rho_{\theta}\simeq 0$ there, support-weighted $\ell_{2}$ losses and metrics on $\omega_{L}$ ignore those pixels. This is the operational counterpart of the (P4) identifiability fact (Appendix~\ref{app:illposed-resolution}) that the data constrains $\omega_{L}$ only on the support of $\rho$: without the mask, the gradient through $\omega_{L}$ on the background is a near-zero coefficient that nonetheless steers the optimizer toward arbitrary Larmor values, contaminating the per-pixel loss landscape.

\subsection{Annealed Fourier Features}
\label{app:method-anneal}
Following the Nerfies frequency-annealing schedule~\citep{park2021nerfies,tancik2020fourier}, the encoded input is
\begin{align}
\label{eq:app-anneal}
\gamma_{\beta}(\mathbf{x})
&\;=\;\Bigl(\,\mathbf{x},\;\bigl\{\,w_{k}(\beta)\,\sin(2^{k}\pi\,\mathbf{x}),\;w_{k}(\beta)\,\cos(2^{k}\pi\,\mathbf{x})\,\bigr\}_{k=0}^{K-1}\Bigr),\\
w_{k}(\beta)
&\;=\;\frac{1-\cos\!\bigl(\pi\,\mathrm{clip}(\beta-k,0,1)\bigr)}{2},\nonumber
\end{align}
where $K=12$ and $\beta\in[0,K]$ is a per-stage progress scalar. We schedule $\beta$ linearly within each stage, $\beta(t)=K\cdot t/T$, with $t$ the within-stage epoch and $T$ the stage length. At $\beta=0$ only the linear coordinate input contributes; at $\beta=K$ all $K$ octaves are fully active and Eq.~\eqref{eq:app-anneal} reduces to the standard Fourier feature encoding of~\citep{tancik2020fourier}. The cosine envelope $w_{k}(\beta)$ activates each band smoothly, avoiding the gradient-shock that abrupt high-frequency activation would induce.

The annealing schedule is the operational counterpart of (P1) (Lemma~\ref{lem:freqdecay}): low spatial frequencies of $\rho$ carry $\Theta(1)$ forward signal, while high frequencies are suppressed with envelope $e^{-k z_{0}}$ up to algebraic factors. Fitting low frequencies first delivers a coarse minimum that is closer to the true support before the high-frequency capacity of the MLP is unlocked. We measure this as a coarse-to-fine vs.~fine-to-coarse curriculum ablation in Section~\ref{sec:experiments}.

\subsection{Multiscale Training Schedule}
\label{app:method-multiscale}
Optimization runs in two stages on a single MLP (parameters carried across stages); only the coordinate grid resolution and the resolution at which $\mathcal{F}_{2}$ is evaluated change between stages. Table~\ref{tab:app-method-schedule} summarizes the schedule used in every reported NeTMY run.

\begin{table}[h]
\centering
\small
\setlength{\tabcolsep}{6pt}
\begin{tabular}{lcccl}
\toprule
Stage & Resolution & Epochs (default) & Learning rate & Notes \\
\midrule
1 & $32\times 32$ & $3000$ & $\eta_{\mathrm{base}}=10^{-3}$ & coarse low-frequency support recovery \\
2 & $64\times 64$ & $7000$ & $0.5\cdot\eta_{\mathrm{base}}$ & fine high-frequency refinement \\
\bottomrule
\end{tabular}
\caption{NeTMY two-stage multiscale schedule. Total $10{,}000$ epochs per measurement. Within each stage, learning rate decays under cosine annealing to $0.01\cdot\eta_{\mathrm{base}}$. Frequency annealing $\beta$ is reset to $0$ at the start of each stage so high-frequency PE bands are re-introduced from coarse to fine within each resolution.}
\label{tab:app-method-schedule}
\end{table}

The same MLP weights are carried across stages with no re-initialization, so the fitted parameters at the end of Stage~1 act as a structured initialization for Stage~2. Coordinates are recreated at the new resolution, and the differentiable forward operator $\mathcal{F}_{2}$ is evaluated at that resolution using a precomputed FFT kernel cache. Per-instance wall-clock on a single A6000 GPU ranges from $273$~s (median) to $\sim\!780$~s (untrimmed mean) over the main $512$-sample benchmark; runtime is reported in Appendix~\ref{app:exp-runtime}.

\subsection{Loss Terms}
\label{app:method-losses}
The full per-measurement objective of Eq.~\eqref{eq:method-loss} is the sum of a canonical fidelity, two standard regularizers, and two NeTMY-specific physics losses. Below we give the explicit form and the role of each term; default loss weights are listed in Table~\ref{tab:app-method-losses}.

\begin{itemize}[leftmargin=*]
\item \textbf{Spectrum fidelity} $\mathcal{D}$: the pixelwise log-MSE on max-normalized noise maps as defined in Eq.~\eqref{eq:fidelity}. This is the sole data-fidelity term.
\item \textbf{Sparsity} $\norm{\rho_{\theta}}_{1}=\sum_{\mathbf{r}}\abs{\rho_{\theta}(\mathbf{r})}$: encourages point-like sources.
\item \textbf{Anisotropic total variation} $\mathrm{TV}(\rho_{\theta})$ as defined after Eq.~\eqref{eq:invprob}: smooths small-scale noise without penalizing sparse peaks.
\item \textbf{Mean-normalized noise-map loss} $\mathcal{R}_{\mathrm{nm}}=\mathrm{MSE}\bigl(N_{\mathrm{pred}}/\overline{N}_{\mathrm{pred}},\,N_{\mathrm{obs}}/\overline{N}_{\mathrm{obs}}\bigr)$, where $\overline{N}$ denotes the spatial mean of the frequency-summed noise map. This is a scale-stable companion to $\mathcal{D}$ that anchors gradients on the support without the max-normalization peak coupling of (P3). In the $64{\times}64$ stage we replace the log-MSE form of $\mathcal{D}$ with $\mathcal{R}_{\mathrm{nm}}$ as the primary fidelity term, since the log-domain comparison is most useful for early-stage support discovery.
\item \textbf{Direct-density loss} $\mathcal{R}_{\mathrm{ds}}=\mathrm{MSE}\bigl(\rho_{\theta}^{2}/\overline{\rho_{\theta}^{2}},\,N_{\mathrm{obs}}/\overline{N}_{\mathrm{obs}}\bigr)$: a soft, oracle-free support regularizer. The interpretation ``spectrum integral $\propto\rho^{2}$'' is exact only for the coherent point-source operator $\mathcal{F}_{1}$; under the thermal operator $\mathcal{F}_{2}$ used in this paper, the spectrum is linear in $\rho$, so $\mathcal{R}_{\mathrm{ds}}$ is a heuristic that nudges large predicted density values toward bright observed noise pixels rather than an exact functional identity. We retain it because the ablation shows that it accelerates support discovery without harming $\mathcal{F}_{2}$ inversion. No ground-truth density labels enter this term.
\end{itemize}

\begin{table}[h]
\centering
\small
\setlength{\tabcolsep}{6pt}
\begin{tabular}{lcl}
\toprule
Term & Weight & Role \\
\midrule
$\mathcal{D}$ & $2.0$ & spectrum fidelity, log-MSE \\
$\mathcal{R}_{\mathrm{nm}}$ & $0.5$ & noise-map fidelity, mean-norm \\
$\mathcal{R}_{\mathrm{ds}}$ & $0.1$ & direct-density proxy, mean-norm \\
$\norm{\rho_{\theta}}_{1}$ & $10^{-2}+10^{-3}$ & sparsity (combined: support + L1) \\
$\mathrm{TV}(\rho_{\theta})$ & $10^{-3}$ & smoothness on small scales \\
\bottomrule
\end{tabular}
\caption{Default loss weights for NeTMY. The two $\ell_{1}$ entries are listed separately because the codebase implements two L1 contributions with different weights (one as a "sparsity" prior across both stages, one as an explicit L1 regularizer that can be ablated independently).}
\label{tab:app-method-losses}
\end{table}

\subsection{Energy-Anchored Scale Correction}
\label{app:method-scale}

After the per-measurement optimization terminates, the predicted density is
rescaled by the energy-ratio anchor of Proposition~\ref{prop:scale}. For a
given inversion operator $\mathcal{F}\in\{\mathcal{F}_{1},\mathcal{F}_{2}\}$,
we forward-evaluate
\begin{equation}
    \widehat{S}=\mathcal{F}(\widehat{\rho},\omega_{L,\theta})
\end{equation}
on the final $64{\times}64$ grid and compute the pre-normalization energies
\begin{equation}
    E_{\mathrm{pred}}=\sum_{\omega,\mathbf{r}}\widehat{S}(\omega,\mathbf{r}),
\qquad
    E_{\mathrm{obs}}=\sum_{\omega,\mathbf{r}}S_{\mathrm{obs}}(\omega,\mathbf{r}).
\end{equation}
We then apply the operator-matched scale correction
\begin{equation}
\label{eq:app-method-scale}
\widehat{\rho}_{\mathrm{scaled}}
=
\alpha_{\mathcal{F}}\widehat{\rho},
\qquad
\alpha_{\mathcal{F}}
=
\begin{cases}
\sqrt{E_{\mathrm{obs}}/E_{\mathrm{pred}}}, & \mathcal{F}=\mathcal{F}_{1},\\[2mm]
E_{\mathrm{obs}}/E_{\mathrm{pred}}, & \mathcal{F}=\mathcal{F}_{2}.
\end{cases}
\end{equation}
The square-root correction for $\mathcal{F}_{1}$ follows from the quadratic
homogeneity $\mathcal{F}_{1}(c\rho,\omega_L)=c^{2}\mathcal{F}_{1}(\rho,\omega_L)$,
whereas the linear correction for $\mathcal{F}_{2}$ follows from
$\mathcal{F}_{2}(c\rho,\omega_L)=c\mathcal{F}_{2}(\rho,\omega_L)$; the derivation
is given in Appendix~\ref{app:scale-s1}.
This correction mainly affects density-scale-dependent quantities such as
density MSE. The primary localization-oriented metrics reported in
Section~\ref{sec:experiments}, including GMSD, Hungarian F1, and Sliced-WD, are
insensitive or only weakly sensitive to the absolute density scale. Density MSE
is therefore reported as a secondary metric.
We apply the correction as a one-shot post-processing step rather than as a soft
penalty inside Eq.~\eqref{eq:method-loss}. This avoids reweighting the
$\mathcal{D}$-gradient by a scalar that depends on the current iterate, which
would otherwise couple the support-discovery dynamics to the absolute amplitude
estimate. Within each fixed operator setting, all methods are evaluated with the
same operator-matched scale convention, so comparisons across parameterizations
are not affected by method-specific scale-correction choices.

\subsection{Filtering View of the Update}
\label{app:method-geometry}
We give the formal version of the filtering identity Eq.~\eqref{eq:method-geometry} that motivates the optimization-geometry interpretation in Section~\ref{sec:method-geometry}.

\begin{lemma}[Jacobian-induced filtering of the density update]
\label{lem:app-filter}
Let $f_{\theta}:\Omega\to\mathbb{R}_{\geq 0}$ be a coordinate field parameterized by $\theta\in\mathbb{R}^{P}$, write $\rho_{\theta}=f_{\theta}$ as a vector in $\mathbb{R}^{|\Omega|}$, and let $\mathcal{L}(\rho)$ be a differentiable scalar loss on density. A vanilla first-order step on $\theta$ at learning rate $\eta>0$, $\theta_{t+1}=\theta_{t}-\eta\,\nabla_{\theta}\mathcal{L}(\rho_{\theta_{t}})$, induces
\begin{equation}
\label{eq:app-filter}
\Delta\rho\;=\;\rho_{\theta_{t+1}}-\rho_{\theta_{t}}\;=\;-\eta\,J_{\theta}J_{\theta}^{\top}\,\nabla_{\rho}\mathcal{L}(\rho_{\theta_{t}})\;+\;O\!\bigl(\eta^{2}\,\norm{J_{\theta}}_{\mathrm{op}}^{2}\,\norm{\nabla_{\rho}\mathcal{L}}^{2}\,\norm{\partial_{\theta}^{2}f_{\theta}}_{\mathrm{op}}\bigr),
\end{equation}
where $J_{\theta}=\partial f_{\theta}/\partial\theta\in\mathbb{R}^{|\Omega|\times P}$ is the parameterization Jacobian and $\norm{\partial_{\theta}^{2}f_{\theta}}_{\mathrm{op}}$ is the bilinear operator norm of the second-derivative tensor of $f_{\theta}$. The first-order term is the raw density gradient $\nabla_{\rho}\mathcal{L}$ filtered by the positive semidefinite kernel $G_{\theta}=J_{\theta}J_{\theta}^{\top}\in\mathbb{R}^{|\Omega|\times|\Omega|}$.
\end{lemma}
\begin{proof}
Let $g=\nabla_{\rho}\mathcal{L}(\rho_{\theta_{t}})$. Chain rule gives $\nabla_{\theta}\mathcal{L}=J_{\theta}^{\top}g$, hence $\Delta\theta=-\eta\,J_{\theta}^{\top}g$ and $\norm{\Delta\theta}\leq\eta\,\norm{J_{\theta}}_{\mathrm{op}}\,\norm{g}$. Taylor expansion of $f_{\theta}$ in $\theta$ yields $\rho_{\theta_{t}+\Delta\theta}=\rho_{\theta_{t}}+J_{\theta}\Delta\theta+\tfrac{1}{2}\,\partial_{\theta}^{2}f_{\theta}[\Delta\theta,\Delta\theta]+O(\norm{\Delta\theta}^{3})$, which combined with the bound on $\norm{\Delta\theta}^{2}$ gives Eq.~\eqref{eq:app-filter}.
\end{proof}
NeTMY uses AdamW rather than vanilla gradient descent, so this lemma should be read as the leading-order geometry of an unpreconditioned step; AdamW inserts an adaptive parameter-space preconditioner $P_{t}$ in place of the identity, but the resulting density update is still pushed forward by $J_{\theta}$ into the same image-space subspace, so the qualitative filtering picture is preserved.

\textbf{Consequence for ill-posedness.}
The filtering kernel
\begin{equation}
G_{\theta}=J_{\theta}J_{\theta}^{\top}
\end{equation}
is positive semidefinite by construction. It acts as a structurally coupled filtering operator across pixels. Writing the Jacobian column functions as
\begin{equation}
\psi_p(\mathbf r)
=
\frac{\partial f_\theta(\mathbf r)}{\partial \theta_p},
\end{equation}
we have, for any raw density gradient $g(\mathbf r)$,
\begin{equation}
(G_\theta g)(\mathbf r)
=
\sum_{p=1}^{P}
\psi_p(\mathbf r)\langle \psi_p,g\rangle
\in
\operatorname{span}\{\psi_1,\ldots,\psi_P\}.
\end{equation}
Hence $G_\theta$ maps the raw pixel-wise gradient into the structured function space allowed by the coordinate parameterization, rather than allowing an arbitrary free pixel-wise update. Equivalently,
\begin{equation}
\operatorname{rank}(G_\theta)
=
\operatorname{rank}(J_\theta J_\theta^\top)
\le
\min(|\Omega|,P),
\end{equation}
so the induced density update is constrained to a structured low-dimensional subspace of the full pixel space.

For a coordinate MLP with annealed Fourier features, the functions $\psi_p(\mathbf r)$ are smooth in the early stages because high-frequency Fourier bands are suppressed in the encoding $\gamma_\beta$. If
\begin{equation}
w_k(\beta)\approx 0,
\qquad k>k_\beta ,
\end{equation}
then the effective bandwidth of the encoding is approximately
\begin{equation}
B_\beta\sim 2^{k_\beta}\pi .
\end{equation}
Consequently, the Jacobian columns are dominated by modes below $B_\beta$, and
\begin{equation}
\widehat{\psi_p}(\boldsymbol\xi)\approx 0,
\qquad
|\boldsymbol\xi|>B_\beta .
\end{equation}
It follows that
\begin{equation}
\widehat{G_\theta g}(\boldsymbol\xi)
=
\sum_{p=1}^{P}
\langle \psi_p,g\rangle
\widehat{\psi_p}(\boldsymbol\xi)
\approx 0,
\qquad
|\boldsymbol\xi|>B_\beta .
\end{equation}
Thus $G_\theta$ initially behaves as a low-pass smoothing operator, consistent with the exponential frequency suppression in (P1). As the annealing parameter $\beta$ increases, higher Fourier bands are gradually activated, the effective bandwidth of the Jacobian columns grows, and the spectrum of $G_\theta$ becomes progressively richer.
In this sense, the coordinate parameterization redistributes a single-pixel density perturbation into a smooth, spatially coupled update during early training, and only later allows sharper, higher-frequency corrections.
A sharp center spike in $\nabla_{\rho}\mathcal{L}$ (induced by (P2)+(P3) on a uniform initialization) is therefore mapped to a spatially distributed image-space update whose center-to-outer ratio is bounded by the corresponding ratio of $G_{\theta}$ rather than by $\nabla_{\rho}\mathcal{L}$ itself. We measure this directly in Section~\ref{sec:experiments}: the realized iter-$0$ NeTMY update $|\Delta\rho|$ does not exhibit a singular center spike, in contrast to the iter-$0$ free-density gradient under the same forward operator and same uniform initialization.

\textbf{Comparison to free-density and quasi-Newton solvers.} For Tikhonov and ADMM, $\rho$ is the optimization variable and the parameter-space update is a (sub-)gradient step, so the effective parameterization Jacobian is the identity ($J=I$, $G_{\theta}=I$) and Eq.~\eqref{eq:app-filter} reduces to $\Delta\rho=-\eta\,\nabla_{\rho}\mathcal{L}$: the raw density gradient is executed verbatim, including the (P2) center bias and the (P3) self-reinforcing peak coupling. L-BFGS also optimizes $\rho$ directly, but replaces the gradient direction with $-H_{t}^{-1}\nabla_{\rho}\mathcal{L}$ via a low-rank approximate-inverse-Hessian $H_{t}^{-1}$ assembled from past gradient differences; this is a parameter-space preconditioner orthogonal to the parameterization-Jacobian filter analyzed here, and provides an independent route to escape the centered-collapse basin (consistent with the L-BFGS center-mass ratio of $0.081$ in Section~\ref{sec:experiments}). For GaussianSplat (Gaussian splats), $G_{\theta}$ is rank-bounded by the number of primitives times the dimension of the per-primitive parameter space, which is small, so $G_{\theta}$ is a strongly localized smoother that cannot represent dense distributions but does limit isolated single-pixel updates on sparse scenes. The empirical ranking documented in Section~\ref{sec:experiments} is consistent with this filtering view: among solvers whose update geometry is dominated by Eq.~\eqref{eq:app-filter}, methods with richer, smoother $G_{\theta}$ are less attracted to the (P2)+(P3) center collapse, with NeTMY (smooth, fully connected MLP across the grid) outperforming GaussianSplat (localized splats) which in turn outperforms Tikhonov / ADMM (no parameterization). We do not claim that this filtering view is the only mechanism at play, only that it is sufficient to explain the observed iter-$0$ image-space update structure and the ranking direction within the gradient-step solver family.

\subsection{Algorithm Summary}
\label{app:method-algo}
\begin{algorithm}[h]
\caption{NeTMY per-measurement optimization (operator $\mathcal{F}\in\{\mathcal{F}_{1},\mathcal{F}_{2}\}$)}
\label{alg:netmy}
\begin{algorithmic}[1]
\State \textbf{Input:} measured spectrum $S_{\mathrm{obs}}$, inversion operator $\mathcal{F}\in\{\mathcal{F}_{1},\mathcal{F}_{2}\}$, total epochs $T_{\mathrm{tot}}$, base lr $\eta_{0}$
\State Initialize MLP parameters $\theta\sim$ default init
\For{stage $s\in\{1,2\}$ with resolution $R_{s}\in\{32,64\}$ and epochs $T_{s}$, lr $\eta_{s}$}
    \State Build coordinate grid $\{\mathbf{x}_{ij}\}$ at resolution $R_{s}$
    \State Precompute FFT kernels for $\mathcal{F}$ at $R_{s}$
    \For{$t=1,\dots,T_{s}$}
        \State $\beta\leftarrow K\cdot t/T_{s}$
        \State $(\rho_{\theta},\omega_{L,\theta})\leftarrow f_{\theta}(\mathbf{x}_{ij};\beta)$ \Comment{Eq.~\eqref{eq:method-density-head}, Eq.~\eqref{eq:app-larmor-mask}}
        \State $\widehat{S}\leftarrow\mathcal{F}(\rho_{\theta},\omega_{L,\theta})$ \Comment{Eq.~\eqref{eq:s1s2}}
        \State $\mathcal{L}\leftarrow$ Eq.~\eqref{eq:method-loss} with weights from Table~\ref{tab:app-method-losses}
        \State $\theta\leftarrow$ AdamW step with cosine-annealed lr, gradient clip
    \EndFor
\EndFor
\State $E_{\mathrm{pred}}\leftarrow\sum_{\omega,\mathbf{r}}\mathcal{F}(\rho_{\theta},\omega_{L,\theta});\;E_{\mathrm{obs}}\leftarrow\sum_{\omega,\mathbf{r}}S_{\mathrm{obs}}$
\State $\alpha_{\mathcal{F}}\leftarrow E_{\mathrm{obs}}/E_{\mathrm{pred}}$ if $\mathcal{F}=\mathcal{F}_{2}$, else $\sqrt{E_{\mathrm{obs}}/E_{\mathrm{pred}}}$
\State $\widehat{\rho}\leftarrow\alpha_{\mathcal{F}}\,\rho_{\theta}$ \Comment{Eq.~\eqref{eq:app-method-scale}}
\State \Return $(\widehat{\rho},\omega_{L,\theta})$
\end{algorithmic}
\end{algorithm}

\section{Experimental Details}
\label{app:experiments}

This appendix expands Section~\ref{sec:experiments} along five axes: dataset construction (Appendix~\ref{app:exp-datasets}), baseline implementation (Appendix~\ref{app:exp-baselines}), metric definitions (Appendix~\ref{app:exp-metrics}), the full $\mathcal{F}_{2}/\mathcal{F}_{2}$ matched-operator results and per-class breakdowns (Appendix~\ref{app:exp-matched}), the mechanism aggregate (Appendix~\ref{app:exp-mechanism}), the full ablation sweep including hyperparameter sweeps (Appendix~\ref{app:exp-ablation}), the supervised-baseline distribution-shift study (Appendix~\ref{app:exp-supervised}), the four-step real-data protocol (Appendix~\ref{app:exp-realdata-full}), runtime and failure modes (Appendix~\ref{app:exp-runtime}, Appendix~\ref{app:exp-failure}). Implementation hyperparameters that are NeTMY-specific (architecture, optimizer, multiscale schedule, loss weights, scale correction) are in Appendix~\ref{app:method}.

\subsection{Datasets}
\label{app:exp-datasets}

\textbf{Cross-fidelity benchmark (Section~\ref{sec:exp-benchmark}, 512 samples).} The cross-fidelity benchmark uses spectra generated by the source-side direct simulator $\mathcal{F}_{3}$ (Appendix~\ref{app:operator-direct}) on a $64\times 64$ grid at standoff $z_{0}=20$\,nm and grid spacing $20$\,nm, with sensor noise $\sigma_{\varepsilon}^{2}$ matched to per-sample dynamic range. Source counts and minimum separations are drawn from eight scene classes that cover the regimes most relevant to NV relaxometry: \textit{few/close}, \textit{few/medium}, \textit{few/far}, \textit{medium/close}, \textit{medium/medium}, \textit{medium/far}, \textit{many/close}, \textit{many/far}. The Larmor field $\omega_{L}\in[1.5,2.5]$\,GHz is evaluated on a $50$-frequency grid and is constant within each ground-truth source. Sample IDs follow the pattern $\langle\text{seq}\rangle\_\langle\text{class}\rangle$ used internally for cross-referencing. Each method is run on the same $512$ samples under both $\mathcal{F}_{1}$ and $\mathcal{F}_{2}$ for inversion (so the table reports $4\times 2=8$ method-operator conditions).

\textbf{Matched-operator benchmark (Section~\ref{sec:exp-benchmark}, 512 samples).} The matched-operator benchmark uses spectra generated and inverted by $\mathcal{F}_{2}$, on the same physical setup as the cross-fidelity benchmark. The samples cover the same scene classes and are used as a broader-coverage view that includes amortized and untrained-prior baselines (DeepDecoder, HybridNeTMY) for which the $\mathcal{F}_{3}$ form is not implemented. The matched-operator setup is intentionally an inverse-crime configuration~\citep{kaipio2007statistical}; we report it only as a complementary view, with the cross-fidelity benchmark as the primary evidence.

\textbf{Real-data dataset (Section~\ref{sec:exp-realdata}).} The $\alpha$-RuCl$_{3}$ dataset is the publicly released NV-relaxometry data from~\citet{kumar2024room}, comprising $T_{1}$-with-flake and $T_{1}$-without-flake measurements at $8$ NVs (NV2, NV3, NV4, NV6, NV7, NV8, NV9, NV10) at the same magnetic-field setting, plus a magnetic-field-dependent dataset on a separate NV (used only for spectral calibration). The crystallographic spin density per layer is $\rho_{0}=6.45\times 10^{18}$\,m$^{-2}$ and the effective magnetic moment is $\mu_{\mathrm{eff}}=2.31\,\mu_{B}$~\citep{kumar2024room}. The SRIM-implied NV-depth prior is $14\pm 5$\,nm, derived from the reported ion implantation energy of $9.8$\,keV and SRIM ion-range/straggle simulation.

\subsection{Baseline Implementation}
\label{app:exp-baselines}

All baselines minimize Eq.~\eqref{eq:invprob} on the same coordinate grid as NeTMY, using the same forward operator and the same observed spectrum, but differ in how $\rho$ is parameterized and which auxiliary regularizers are added. We give compact algorithmic settings for each.

\textbf{Tikhonov.} Free-density $\rho\in\mathbb{R}_{\geq 0}^{|\Omega|}$, optimized by Adam~\citep{kingma2014adam} with learning rate $5\!\times\!10^{-3}$, weight decay $10^{-5}$, and gradient clipping at $1$ for $5000$ epochs. Regularizers are $\ell_{2}$ (weight $10^{-3}$) and TV (weight $10^{-3}$), matching the NeTMY weights up to the absent NeTMY-specific physics losses. The scale-correction post-processing of Eq.~\eqref{eq:app-method-scale} is applied to all methods identically for paired comparison.

\textbf{ADMM.} Variable-splitting on $\rho=z$ with augmented-Lagrangian penalty $\mu=10^{-3}$, box constraints $[0,\rho_{\mathrm{max}}]$ via projection, and proximal $\ell_{1}$ shrinkage with weight $10^{-2}$. The data-fidelity step uses Adam at learning rate $5\!\times\!10^{-3}$ for $30$ inner iterations per outer cycle, run for $200$ outer cycles. Stopping criteria are $\norm{\rho-z}_{\infty}<10^{-3}$ or maximum iterations.

\textbf{GaussianSplat.} Explicit set of $K=64$ Gaussian primitives with parameters $(\boldsymbol{\mu}_{k},\boldsymbol{\sigma}_{k},a_{k})\in\mathbb{R}^{2}\times\mathbb{R}_{>0}^{2}\times\mathbb{R}_{\geq 0}$, where $\boldsymbol{\sigma}_{k}$ is the per-primitive width vector (distinct from the sigmoid $\sigma(\cdot)$ and from the Gaussian-ansatz width $\sigma_{g}$ used in Section~\ref{sec:exp-realdata}). Prune/split/clone/merge schedule from~\citep{kerbl20233d} adapted to the dipolar forward operator. Optimizer Adam at learning rate $10^{-3}$ for $T=400$ iterations. The number of primitives is dynamic but capped at $128$.

\textbf{DeepDecoder.} Untrained convolutional decoder prior of~\citet{heckel2018deep} with $5$ upsampling stages, channel widths $[128,128,128,128,1]$, and bilinear upsampling. Optimized for $5000$ steps with Adam at learning rate $10^{-3}$. The pixel output is multiplied by $1$ to remain nonnegative; values $<0$ are clamped at zero.

\textbf{U-Net (and U-Net+physics).} Four-level U-Net with channel widths $[32,64,128,256]$, instance normalization, and a final pixelwise softplus, trained on $300$ dense paired samples for $200$ epochs with Adam at learning rate $10^{-3}$. The U-Net+physics variant additionally optimizes a physics-residual loss $\norm{\mathcal{F}_{2}(\widehat{\rho},\omega_{L})-S_{\mathrm{obs}}}_{2}^{2}$ alongside the supervised pixel loss. We report only the supervised-trained model and evaluate on sparse test scenes (which are out-of-distribution by construction); see Appendix~\ref{app:exp-supervised}.

\textbf{GAN-SL+physics.} Generative adversarial network with the same backbone as the U-Net for the generator, a PatchGAN discriminator, and a physics-loss term applied to the generator's output. Training and evaluation protocols mirror the U-Net.

\textbf{HybridNeTMY.} A coordinate-image hybrid in which the input to a U-Net is concatenated with a Fourier-feature encoding of the coordinate grid, mirroring the coordinate-feature input of NeTMY but retaining a CNN architecture and supervised training. Reported as a baseline only because the user community has asked whether the coordinate-feature input alone suffices to produce NeTMY-like behavior under supervised training; the answer (Appendix~\ref{app:exp-supervised}) is that it does not.

\textbf{L-BFGS.} Limited-memory BFGS~\citep{liu1989limited} on free-density $\rho$ with history depth $20$, line-search Wolfe conditions, and $100$ outer iterations. Used only in the mechanism analysis (Section~\ref{sec:exp-mechanism}, Appendix~\ref{app:exp-mechanism}) as a reference quasi-Newton free-density solver.

\subsection{Metric Definitions}
\label{app:exp-metrics}

\textbf{Gradient Magnitude Similarity Deviation (GMSD).} GMSD~\citep{xue2013gradient} measures structural fidelity through the standard deviation of a pointwise gradient-magnitude similarity map. Given the predicted density $\widehat{\rho}$ and ground-truth density $\rho_{\star}$, both rescaled to share the maximum value, define gradient magnitudes $G_{X}=\sqrt{(D_{x}*X)^{2}+(D_{y}*X)^{2}}$ for $X\in\{\widehat{\rho},\rho_{\star}\}$ with the Prewitt operators $D_{x},D_{y}$. The pointwise GMS is $\mathrm{GMS}(\mathbf{r})=(2 G_{\widehat{\rho}}G_{\rho_{\star}}+c)/(G_{\widehat{\rho}}^{2}+G_{\rho_{\star}}^{2}+c)$ with stability constant $c$, and $\mathrm{GMSD}=\sqrt{\mathrm{Var}_{\mathbf{r}}\,\mathrm{GMS}(\mathbf{r})}$. Lower is better.

\textbf{Hungarian F$_{1}$.} A localized peak-matching score: peaks are extracted from $\widehat{\rho}$ and $\rho_{\star}$ by local-maximum filtering above an absolute threshold of $5\%$ of the per-image maximum, paired by the Hungarian algorithm~\citep{kuhn1955hungarian} under a fixed matching radius of $r_{\mathrm{match}}=2$ pixels (corresponding to one half of the Section~\ref{sec:illposed} effective PSF width $w_{\mathrm{psf}}$ at the standoff and characteristic mode used in the cross-fidelity benchmark), and scored as $F_{1}=2\,\mathrm{TP}/(2\,\mathrm{TP}+\mathrm{FP}+\mathrm{FN})$. Higher is better. F$_{1}$ is well-defined only when both reconstruction and ground truth contain at least one supra-threshold peak; we report $F_{1}=0$ for the degenerate case.

\textbf{Sliced Wasserstein Distance (SWD).} The sliced Wasserstein-2 distance~\citep{bonneel2015sliced}, evaluated as the mean over $128$ random projections drawn uniformly from the unit sphere of the $1$-D Wasserstein-2 distance between the two projected mass distributions $P_{\theta\#}\widehat{\rho}/\norm{\widehat{\rho}}_{1}$ and $P_{\theta\#}\rho_{\star}/\norm{\rho_{\star}}_{1}$. Lower is better.

\textbf{Density MSE.} The pixelwise mean square error $\mathrm{MSE}=|\Omega|^{-1}\sum_{\mathbf{r}\in\Omega}(\widehat{\rho}(\mathbf{r})-\rho_{\star}(\mathbf{r}))^{2}$ on the post-scale-corrected reconstruction. Reported as a secondary metric only because the background-dominated character of sparse density fields makes MSE insensitive to localization quality.

\textbf{Masked SSIM.} The structural similarity index~\citep{wang2004image} evaluated on the support of $\rho_{\star}$ (i.e., on pixels where $\rho_{\star}>0$), reported in Section~\ref{sec:exp-ablation} alongside MSE.

\subsection{Matched-Operator Density-MSE Leaderboard}
\label{app:exp-matched}

Table~\ref{tab:app-matched} reports the per-method density-MSE leaderboard on the $512$-sample $\mathcal{F}_{2}/\mathcal{F}_{2}$ matched-operator benchmark, with secondary metrics.

\begin{table}[h]
\centering
\small
\setlength{\tabcolsep}{6pt}
\caption{Density-MSE leaderboard on the $512$-sample $\mathcal{F}_{2}/\mathcal{F}_{2}$ matched-operator benchmark. Density MSE and noise-map MSE are means over $512$ sparse samples. Train time is mean wall-clock on a single A6000 GPU.}
\label{tab:app-matched}
\begin{tabular}{lcccc}
\toprule
Method & Density MSE$\downarrow$ & Noise MSE$\downarrow$ & Train time (s) & Notes \\
\midrule
\textbf{NeTMY (Ours)} & \textbf{0.0037} & \textbf{0.21} & 780 & per-instance, no labels \\
GaussianSplat                  & 0.0067 & 1.09 & 402 & explicit Gaussians \\
Tikhonov                & 0.0117 & 0.81 & 1.5 & free-density + $\ell_{2}$/TV \\
ADMM                    & 0.2555 & 22.94 & 1.1 & free-density + ADMM \\
DeepDecoder             & 0.2558 & 24.25 & 259 & untrained CNN prior \\
HybridNeTMY             & 1.011 & 0.23 & 399 & supervised, dense-trained \\
\bottomrule
\end{tabular}
\end{table}

The matched-operator ranking is consistent with the cross-fidelity benchmark of Table~\ref{tab:main} on density-MSE, but, as discussed in Section~\ref{sec:exp-benchmark}, density-MSE alone is insensitive to centered-collapse failure modes that the cross-fidelity benchmark exposes through Hungarian F$_{1}$ and SWD. The supervised HybridNeTMY's poor MSE reflects the dense-to-sparse distribution-shift failure documented in Appendix~\ref{app:exp-supervised}.

\subsection{Mechanism Aggregate and Per-Method Diagnostics}
\label{app:exp-mechanism}

\textbf{32-run ADMM aggregate sweep.} To isolate the operator effect from solver-specific choices on the centered-collapse phenomenon, we run $32$ ADMM configurations (4 samples $\times$ 8 loss-baseline settings) under both $\mathcal{F}_{1}$ and $\mathcal{F}_{2}$ inversion and report the center-mass ratio at termination. Mean center-mass ratio rises from $0.553\pm 0.344$ under $\mathcal{F}_{1}$ to $0.776\pm 0.369$ under $\mathcal{F}_{2}$ (Table~\ref{tab:app-admm-sweep}), confirming that the more faithful operator drives a stronger center collapse for free-density solvers. The 32-run sweep is ADMM-only by construction; the cross-method evidence is the per-sample comparison of Figure~\ref{fig:mechanism}.

\begin{table}[h]
\centering
\small
\setlength{\tabcolsep}{6pt}
\caption{ADMM-only center-collapse aggregate over $32$ runs ($4$ samples $\times$ $8$ loss-baseline configurations). Higher center-mass ratio means more centered collapse.}
\label{tab:app-admm-sweep}
\begin{tabular}{lcc}
\toprule
Operator & Center-mass ratio (mean$\pm$std) & Sample count \\
\midrule
$\mathcal{F}_{1}$ (scalar) & $0.553\pm 0.344$ & 32 \\
$\mathcal{F}_{2}$ (tensor) & $0.776\pm 0.369$ & 32 \\
\bottomrule
\end{tabular}
\end{table}

\textbf{NeTMY first-step image-space update.} Lemma~\ref{lem:app-filter} predicts that the realized iter-$0$ NeTMY update $|\Delta\rho|$ should be smoothed and structurally coupled across pixels rather than executed verbatim as a singular center spike. We measure this directly on the same many-source sample (sample $024$): the iter-$0$ image-space update $|\Delta\rho|$ is spatially distributed with a center-to-outer ratio of approximately $1.6\times$ (vs $18.29\times$ for the raw density-space gradient). This $\sim\!11\times$ damping is the empirical signature of $G_{\theta}=J_{\theta}J_{\theta}^{\top}$.

\subsection{Component Ablation: Per-Class and Hyperparameter Sweeps}
\label{app:exp-ablation}

The cumulative ablation of Table~\ref{tab:ablation} aggregates over $512$ samples. The largest per-class failure under any ablation is case \textit{medium/medium}, where $-\ell_{1}$ raises density MSE by $\times\!\sim\!2.7$ and degrades SSIM from $0.96$ to $0.82$. The smallest is case  \textit{few/close}, where most ablations leave density MSE essentially unchanged (because a single source is reconstructed correctly by every variant). These sample-level patterns are consistent with the difficulty class.

\textbf{Hyperparameter sweeps.} Table~\ref{tab:app-hp} reports three single-axis sweeps (learning rate, hidden width, PE octave count) on held-out samples. NeTMY is robust to learning rate and hidden width within roughly an order of magnitude of the defaults; PE octave count $K\!=\!12$ is the smallest setting that recovers high-frequency sparse detail without divergent peaks.

\begin{table}[h]
\centering
\small
\setlength{\tabcolsep}{6pt}
\caption{Hyperparameter sweeps over a $64$-sample held-out set. Best per axis in bold.}
\label{tab:app-hp}
\begin{tabular}{lccccc}
\toprule
Axis & \multicolumn{5}{c}{Setting} \\
\midrule
Learning rate $\eta_{0}$ & $10^{-4}$ & $5\!\times\!10^{-4}$ & $\mathbf{10^{-3}}$ & $5\!\times\!10^{-3}$ & $10^{-2}$ \\
\quad density MSE         & 0.0049 & 0.0040 & \textbf{0.0037} & 0.0038 & 0.0067 \\
Hidden width              & $128$ & $256$ & $\mathbf{320}$ & $384$ & $512$ \\
\quad density MSE         & 0.0046 & 0.0040 & \textbf{0.0037} & 0.0038 & 0.0040 \\
PE octaves $K$            & $4$ & $8$ & $\mathbf{12}$ & $16$ & $20$ \\
\quad density MSE         & 0.0083 & 0.0048 & \textbf{0.0037} & 0.0040 & 0.0046 \\
\bottomrule
\end{tabular}
\end{table}

\subsection{Supervised Baselines and Distribution Shift}
\label{app:exp-supervised}

\textbf{Setup.} We train four supervised baselines (GAN-SL+physics, U-Net-SL, U-Net-SL+physics, HybridNeTMY) on dense scenes ($300$ paired samples drawn from a dense-source generator with $30$--$60$ sources per scene) and evaluate on the sparse test scenes used by NeTMY. The dense $\to$ sparse mismatch is severe by design: training scenes have nearly continuous mass distributions while test scenes have $1$--$8$ off-center delta-like sources.

\textbf{Failure observation.} Across the four supervised baselines, the dense-to-sparse evaluation produces three characteristic failure patterns: (i) \textbf{GAN-SL+physics} fails to recover the full sparse support and instead concentrates its prediction on only a few dominant responses, suppressing weaker peaks and producing an incomplete sparse reconstruction. (ii) \textbf{U-Net-SL+physics} shows a smoother failure mode, with low-amplitude diffuse predictions and residual background activation, indicating a tendency toward blurred interpolation rather than faithful sparse peak recovery. (iii) \textbf{U-Net-SL} is the most unstable of the supervised baselines in this experiment, often generating exaggerated peak amplitudes, spurious activations, and broad background artifacts under dense-to-sparse transfer. The \textbf{HybridNeTMY} model, despite its coordinate-feature input, behaves like a supervised CNN under the dense-trained regime: its density MSE on the matched-operator $\mathcal{F}_{2}/\mathcal{F}_{2}$ benchmark is $1.011$ (Table~\ref{tab:app-matched}), the worst of the seven methods. NeTMY is unaffected because it is optimized per measurement without paired density labels.

\textbf{What the failure rules out.} The supervised distribution-shift failure rules out the hypothesis that the NeTMY advantage on sparse scenes is reducible to a coordinate-feature inductive bias: HybridNeTMY shares the coordinate-feature input but, when trained supervised on dense scenes, behaves like the other supervised baselines. The advantage of NeTMY is in the per-measurement optimization-against-physics protocol of Section~\ref{sec:method}, not in the feature space alone.

\subsection{Real-Data Cross-Check: Full Four-Step Protocol}
\label{app:exp-realdata-full}

The Section~\ref{sec:exp-realdata} cross-check runs four steps. We expand each below.

\textbf{Step 1: SRIM-anchored amplitude calibration.} For each NV $i\in\{2,3,4,6,7,8,9,10\}$, we compute the implied NV depth $d_{i}^{\star}$ under each operator by inverting $\Gamma_{\mathrm{Ru}}^{\mathrm{pred}}(d_{i}^{\star})=\Gamma_{\mathrm{Ru}}^{i,\mathrm{meas}}$. Under $\mathcal{F}_{2}$, $1$ NV (NV4) lands in $[9,20]$\,nm at $d^{\star}_{4}=14.49$\,nm; under $\mathcal{F}_{1}$, no NVs land in the window because $\Gamma_{\mathrm{Ru}}^{\mathcal{F}_{1}}(d)>\Gamma_{\mathrm{Ru}}^{\mathrm{meas}}$ for all $d>9$\,nm at the SRIM-prior amplitude. The point amplitude ratio at $d=14.5$\,nm is $\Gamma^{\mathcal{F}_{1}}/\Gamma^{\mathcal{F}_{2}}=11{,}353.5/54.0\approx 210\times$.

\textbf{Step 2: Power-law analysis.} The independent quantity $\Gamma_{\mathrm{int}}\propto d^{-\alpha}$ with $\alpha=1.06\pm 0.22$ provided by~\citet{kumar2024room} lets us eliminate $d$ and convert each operator's depth scaling into an internal-rate scaling: $\mathcal{F}_{2}$ predicts $b_{\mathrm{pred}}=4/\alpha=3.77$, $\mathcal{F}_{1}$ predicts $b_{\mathrm{pred}}=2/\alpha=1.89$. The measured slope is $b_{\mathrm{exp}}=1.43\pm 0.21$. Slope alone is closer to the $\mathcal{F}_{1}$ prediction; we therefore explicitly report this axis as ambiguous, since slope and intercept are independent.

\textbf{Step 3: Spectral calibration.} Magnetic-field-dependent $T_{1}$ measurements at $B\in\{2,313,383.5\}$\,Gauss (single NV) calibrate the spin-bath correlation time $\tau_{c}=0.41$\,ns and bath linewidth $\gamma_{\mathrm{bath}}=2.45$\,GHz. Both operators share the Lorentzian spectral response, so this calibration is operator-neutral; we report it for completeness.

\textbf{Step 4: Loss-landscape identifiability.} On the Gaussian density ansatz $\rho(\mathbf{r};A,\sigma_g)=A\exp(-\norm{\mathbf{r}}^{2}/2\sigma_g^{2})$, where $\sigma_g$ denotes the Gaussian width, we evaluate the log-MSE loss $L(A,\sigma_g)=|N_{\mathrm{NV}}|^{-1}\sum_{i}(\log\Gamma_{\mathrm{Ru}}^{\mathrm{pred}}(A,\sigma_g,d_{i})-\log\Gamma_{\mathrm{Ru}}^{i,\mathrm{meas}})^{2}$ on a $90\times 90$ grid in $(\log_{10}A,\sigma_g)$. The Hessian condition number at the minimum is $\kappa_{\mathcal{F}_{2}}=931$ and $\kappa_{\mathcal{F}_{1}}=301{,}139$ (ratio $\sim\!323\times$). The $\mathcal{F}_{2}$ landscape is a parabolic bowl; the $\mathcal{F}_{1}$ landscape is a degenerate valley along $A^{2}\sigma_g^{2}=\mathrm{const}$, which is the algebraic consequence of $\Gamma\propto A^{2}$ in $\mathcal{F}_{1}$ and is non-identifiable on this dataset.

\textbf{Read-back.} Two of three axes (depth-amplitude, conditioning) disfavor $\mathcal{F}_{1}$; the third (slope) is ambiguous. We frame the result as a multi-axis consistency check, not as a validation, because the dataset is $8$ NVs with a $1$-D depth axis and a Gaussian density ansatz, not a spatially resolved reconstruction benchmark.

\subsection{Runtime}
\label{app:exp-runtime}

NeTMY's per-instance wall-clock on the $512$-sample matched-operator benchmark on a single NVIDIA A6000 GPU is long-tailed: median $273$\,s, trimmed mean (drop top $5$ samples) $426$\,s, untrimmed mean $780$\,s. The top-$5$ samples are dense \textit{many/close} configurations where the $64\times 64$ stage requires more iterations to converge on the full source set. Free-density baselines (Tikhonov $1.55$\,s, ADMM $1.10$\,s) are roughly $200$--$700\times$ faster on the same hardware; GaussianSplat ($402$\,s) and DeepDecoder ($259$\,s) are within an order of magnitude of NeTMY. In the regime that motivated this work (per-instance reconstruction without paired density labels), the relevant comparison is to the supervised amortized baselines, which are fast at inference but require a labeled training set that is generally not available for NV-relaxometry samples; in regimes where labels are abundant, amortized methods are the appropriate choice.

\subsection{Failure Modes}
\label{app:exp-failure}

We document three known failure modes that the experimental analysis surfaces.

\textbf{Cross-shaped artifacts on dense scenes.} As shown in Figure~\ref{fig:Failure1}, on samples in the \textit{many/medium} class with $\sim\!30$--$40$ sources spaced at the resolution limit, NeTMY occasionally produces axis-aligned cross-shaped artifacts in the reconstructed density. The artifact morphology is consistent with the (P4) anisotropic-explanation prediction of Section~\ref{sec:illposed}: when many sources lie within $w_{\mathrm{psf}}$ of each other, the optimizer converges to a low-cost density that places mass along low-cost axis-aligned tracks rather than at the true source positions. Anti-artifact regularizers (isotropic TV, Laplacian, Fourier-axis penalties) reduce visible cross artifacts but slightly worsen aggregate density MSE; we keep the standard anisotropic TV in the main runs.

\textbf{High-frequency leakage on small clusters.} The medium/medium row of Figure~\ref{fig:Failure2} shows that NeTMY occasionally leaks small-cluster mass into nearby pixels. This is the high-frequency tail of the annealed PE schedule: at full $\beta=K$ the network can fit non-source-supported pixels, and the $\ell_{1}$/gate combination is the only mechanism preventing this. On samples with $\sim\!8$--$15$ closely spaced sources the leakage is visible; on simpler sparse scenes it is suppressed.

\textbf{Centered minimum trapping under random initialization.} For $\sim\!5\%$ of the matched-operator samples, NeTMY's stage-1 minimum is in the centered-collapse basin documented in Section~\ref{sec:exp-mechanism}. Stage-2 refinement at higher resolution and lower learning rate escapes this basin via the $\beta$-schedule reset, but the fraction of samples in which this happens is sensitive to the random initialization of the MLP and to stage-1 epoch count. We did not observe any systematic correlation with scene class.

\begin{figure}[h]
    \centering
    \includegraphics[width=0.6\textwidth]{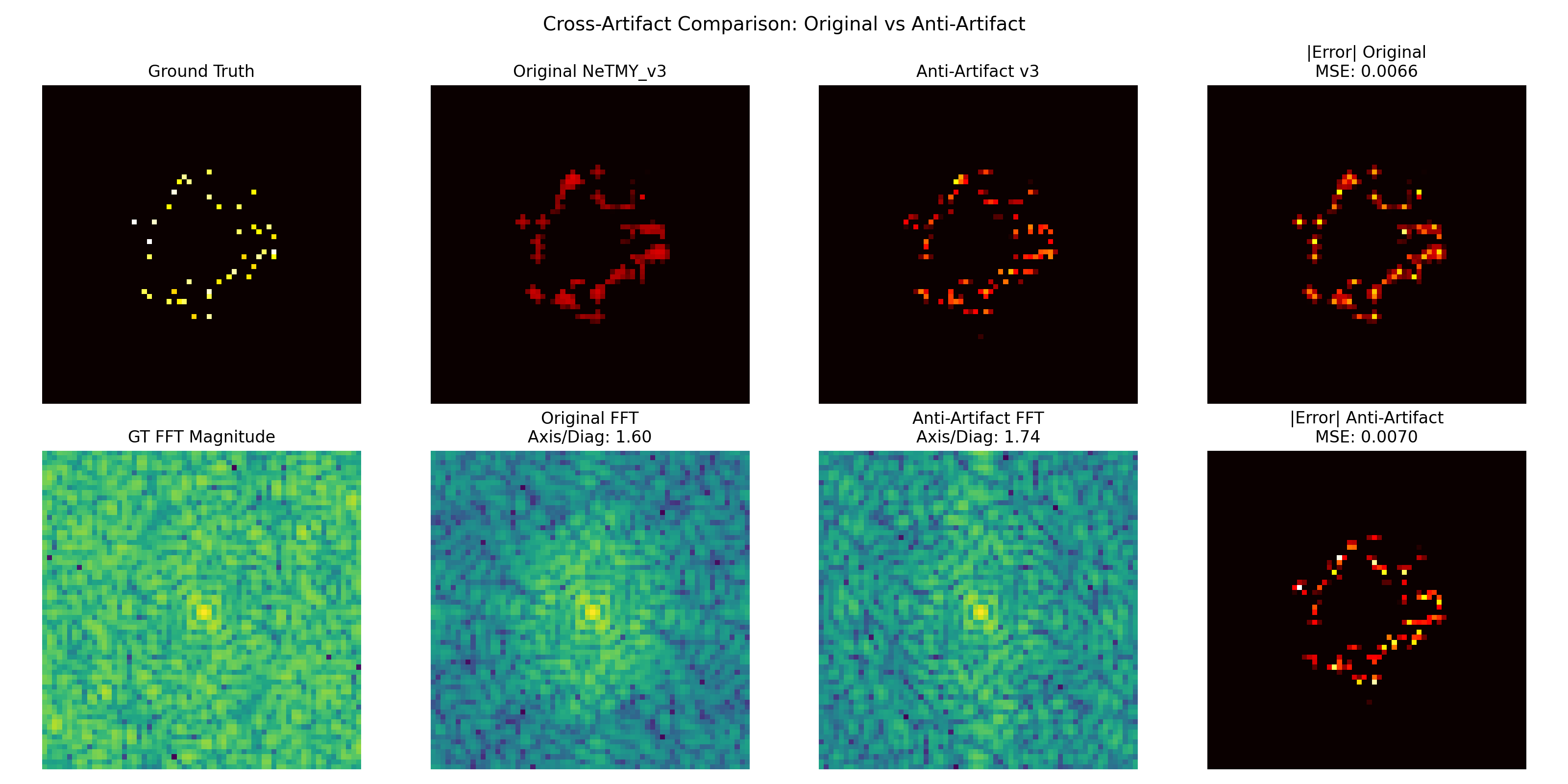}
    \caption{The reconstruction captures the broad support of the active region but develops strong cross-like, axis-aligned artifacts around high-intensity locations. The anti-artifact variant substantially suppresses these artifacts, albeit at the cost of slightly degraded reconstruction quality.}
    \label{fig:Failure1}
\end{figure}

\begin{figure}[h]
    \centering
    \includegraphics[width=0.5\textwidth]{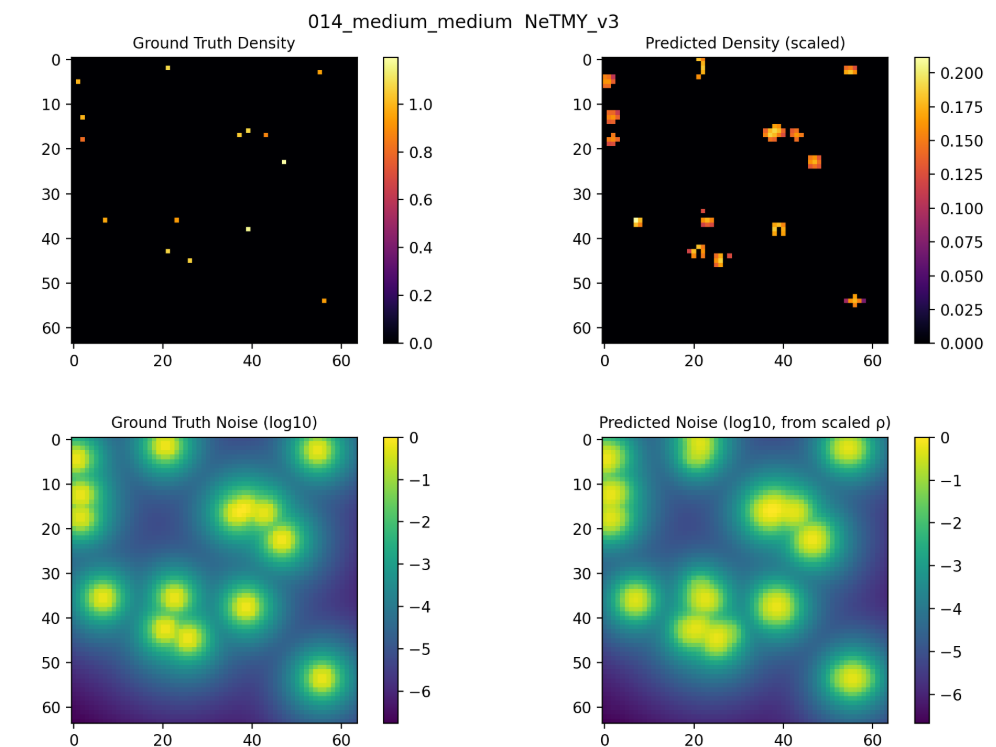}
    \caption{High-frequency leakage on a medium/medium sample. The predicted density spreads small-cluster mass into neighboring pixels beyond the true source locations, a leakage artifact driven by the high-frequency tail of the annealed PE schedule. The $\ell_1$/gate combination partially suppresses this effect, but residual leakage remains visible given the $\sim$8--15 closely spaced sources in this scene.}
    \label{fig:Failure2}
\end{figure}

\begin{figure}[H]
    \centering
    \includegraphics[width=0.5\textwidth]{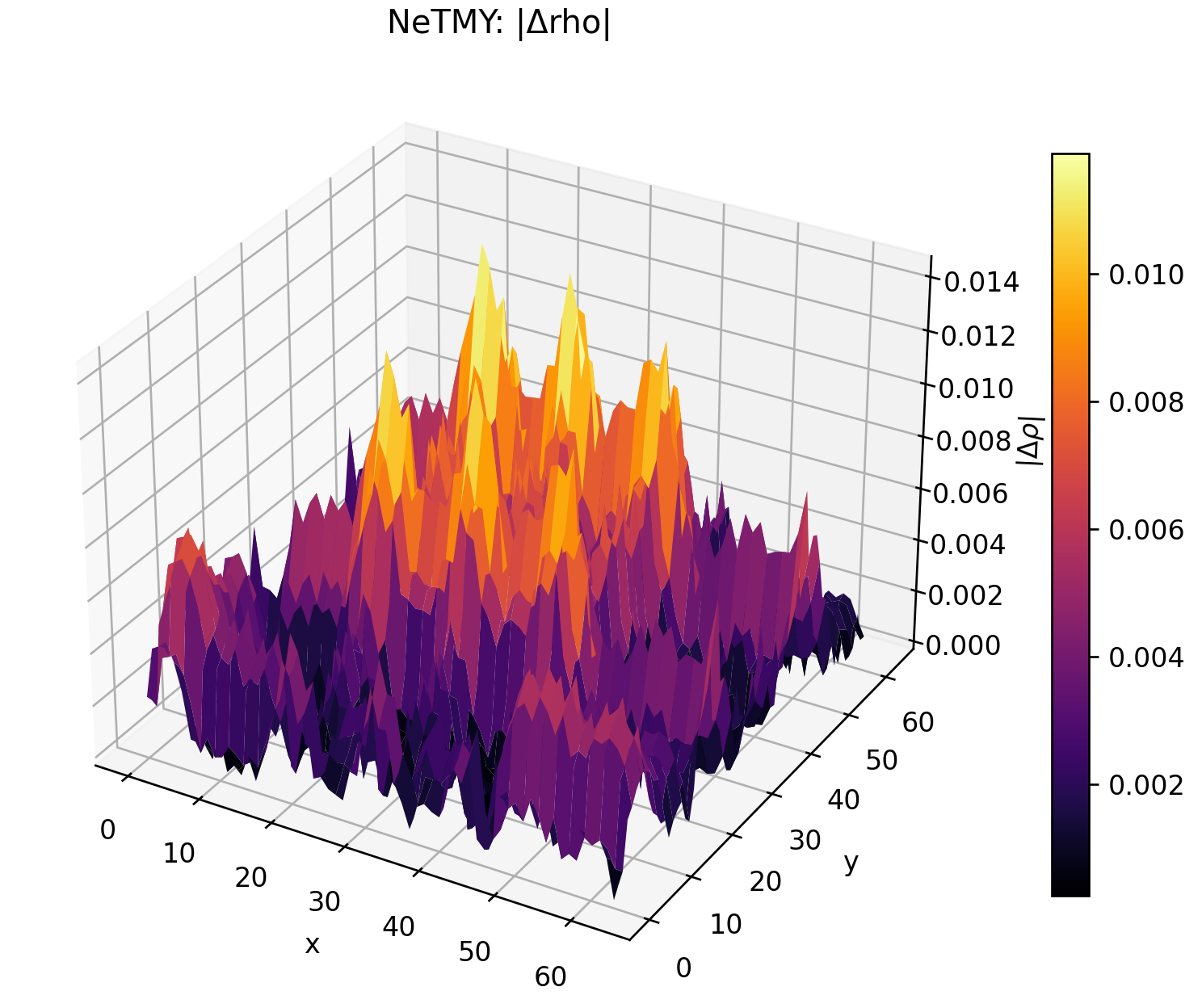}
    \caption{Realized first-step image-space update $|\Delta\rho|$ produced by NeTMY, where $\Delta\rho=\rho_{1}-\rho_{0}$. The surface does \emph{not} exhibit a singular center spike, in contrast to the iter-0 free-density gradient under the same forward operator (Section~\ref{sec:exp-mechanism}). This is the empirical signature of the Lemma~\ref{lem:app-filter} filtering kernel $G_{\theta}=J_{\theta}J_{\theta}^{\top}$ redistributing the raw density-space gradient through the parameterization.\label{fig:app-netmy-delta}}
\end{figure}

\subsection{Qualitative Examples}
\label{app:exp-examples}

This subsection presents additional per-sample reconstruction grids that complement the cross-fidelity benchmark of Section~\ref{sec:exp-benchmark} and the matched-operator benchmark of Appendix~\ref{app:exp-matched}. Figures~\ref{fig:QTD-008}--\ref{fig:QTD-018} show four cross-fidelity samples ($\mathcal{F}_{3}$-generated, inverted under both $\mathcal{F}_{1}$ and $\mathcal{F}_{2}$) and illustrate the operator-fidelity reshaping discussed in Section~\ref{sec:exp-benchmark}. Figures~\ref{fig:QTD-Shallow} and~\ref{fig:QTD-Physical} show four matched-operator samples each under $\mathcal{F}_{1}/\mathcal{F}_{1}$ and $\mathcal{F}_{2}/\mathcal{F}_{2}$, respectively, drawn from the matched-operator benchmark of Appendix~\ref{app:exp-matched}. Within each figure, columns correspond to (left to right) the ground-truth density, Tikhonov, ADMM, GaussianSplat, and NeTMY reconstructions; the centered iter-$0$ collapse of free-density solvers under $\mathcal{F}_{2}$ is visible as a concentrated central blob in the Tikhonov / ADMM panels, while NeTMY recovers off-center sparse structure consistent with the ground truth.

\begin{figure}[H]
    \centering
    \includegraphics[width=0.8\textwidth]{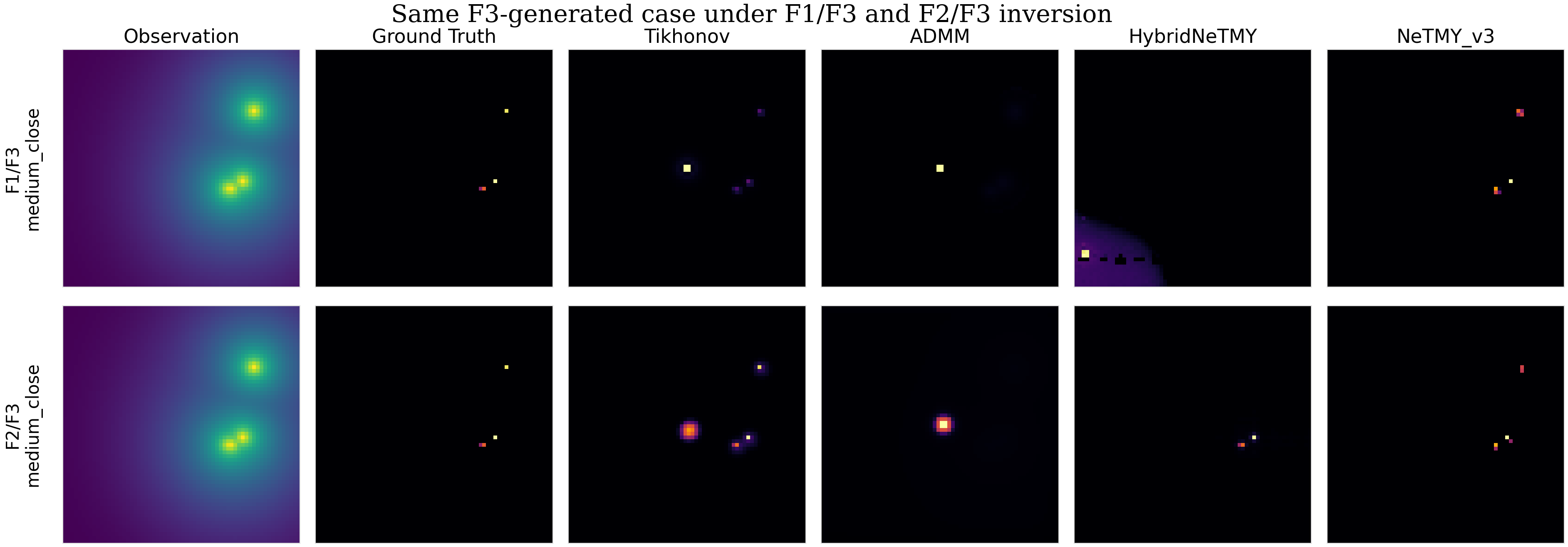}
    \caption{Cross-fidelity sample~$008$ ($\mathcal{F}_{3}$-generated). Reconstructions under $\mathcal{F}_{1}$ and $\mathcal{F}_{2}$ inversion. Free-density solvers (Tikhonov, ADMM) center-collapse under the more faithful $\mathcal{F}_{2}$, while NeTMY recovers the off-center support; GaussianSplat sits between the two regimes. See Section~\ref{sec:exp-benchmark} for aggregate metrics.}
    \label{fig:QTD-008}
\end{figure}

\begin{figure}[H]
    \centering
    \includegraphics[width=0.8\textwidth]{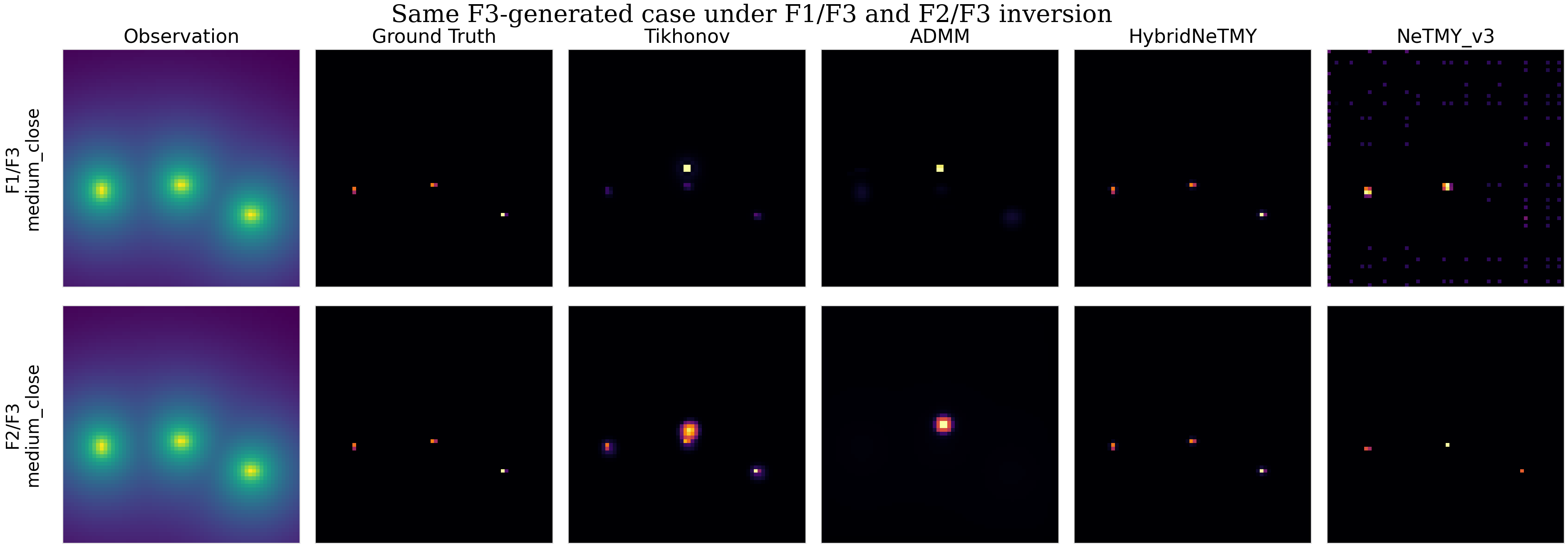}
    \caption{Cross-fidelity sample~$009$ ($\mathcal{F}_{3}$-generated). Same layout as Figure~\ref{fig:QTD-008}: $\mathcal{F}_{1}$ inversion (top row) vs. $\mathcal{F}_{2}$ inversion (bottom row), columns are Tikhonov, ADMM, GaussianSplat, NeTMY. The $\mathcal{F}_{2}$ → $\mathcal{F}_{1}$ ranking flip discussed in Section~\ref{sec:exp-benchmark} is visible at the per-sample level.}
    \label{fig:QTD-009}
\end{figure}

\begin{figure}[H]
    \centering
    \includegraphics[width=0.8\textwidth]{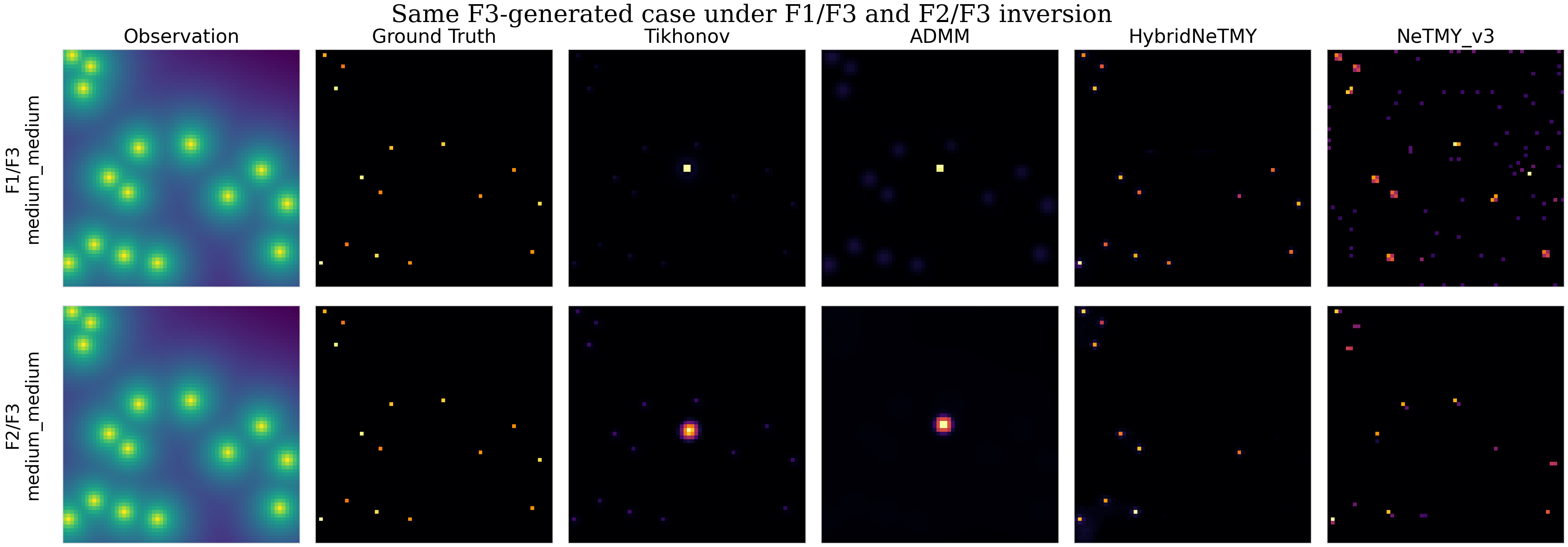}
    \caption{Cross-fidelity sample~$010$ ($\mathcal{F}_{3}$-generated). A \textit{many/close} scene where the (P4) merging pathology and the (P2)+(P3) center-collapse pathology jointly stress the free-density solvers; NeTMY preserves the most peaks within the matching radius of Appendix~\ref{app:exp-metrics}.}
    \label{fig:QTD-010}
\end{figure}

\begin{figure}[H]
    \centering
    \includegraphics[width=0.8\textwidth]{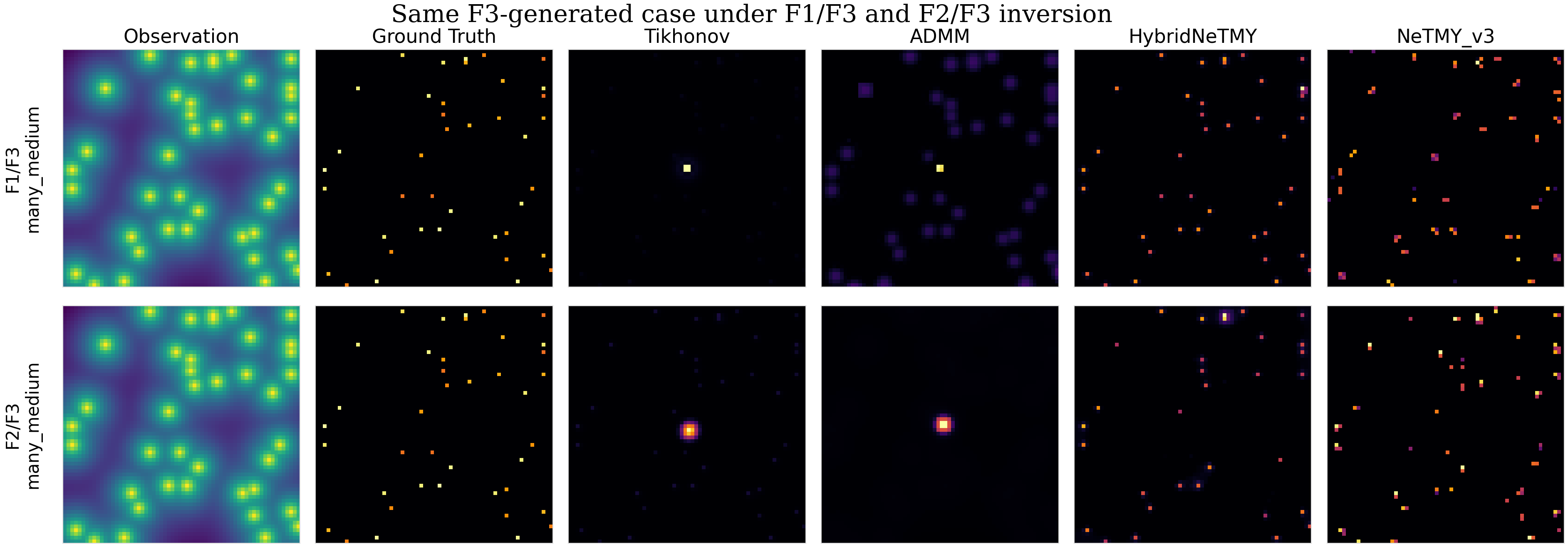}
    \caption{Cross-fidelity sample~$018$ ($\mathcal{F}_{3}$-generated). A \textit{medium/medium} scene; GaussianSplat is competitive on Hungarian F$_{1}$ here while NeTMY remains best on Sliced-Wasserstein, consistent with the table-level rankings of Section~\ref{sec:exp-benchmark}.}
    \label{fig:QTD-018}
\end{figure}

\begin{figure}[H]
    \centering
    \includegraphics[width=1\textwidth]{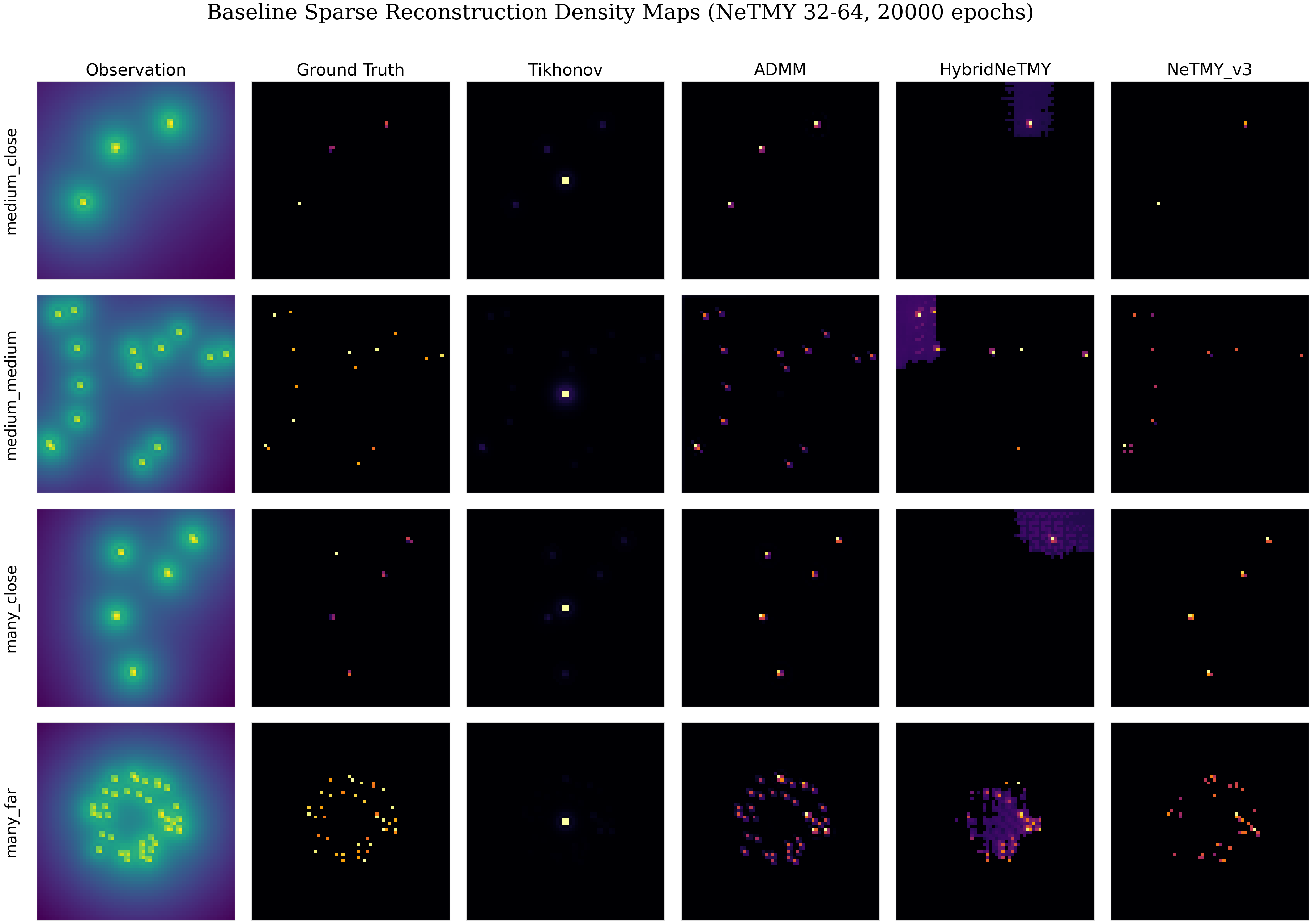}
    \caption{Matched-operator $\mathcal{F}_{1}/\mathcal{F}_{1}$ examples (samples $012$, $016$, $022$, $029$). Under the simplified scalar/coherent operator there is no centering pathology, and ADMM is the lowest-MSE method (Table~\ref{tab:matched}); NeTMY remains visually competitive with cleaner peaks but does not dominate.}
    \label{fig:QTD-Shallow}
\end{figure}

\begin{figure}[H]
    \centering
    \includegraphics[width=1\textwidth]{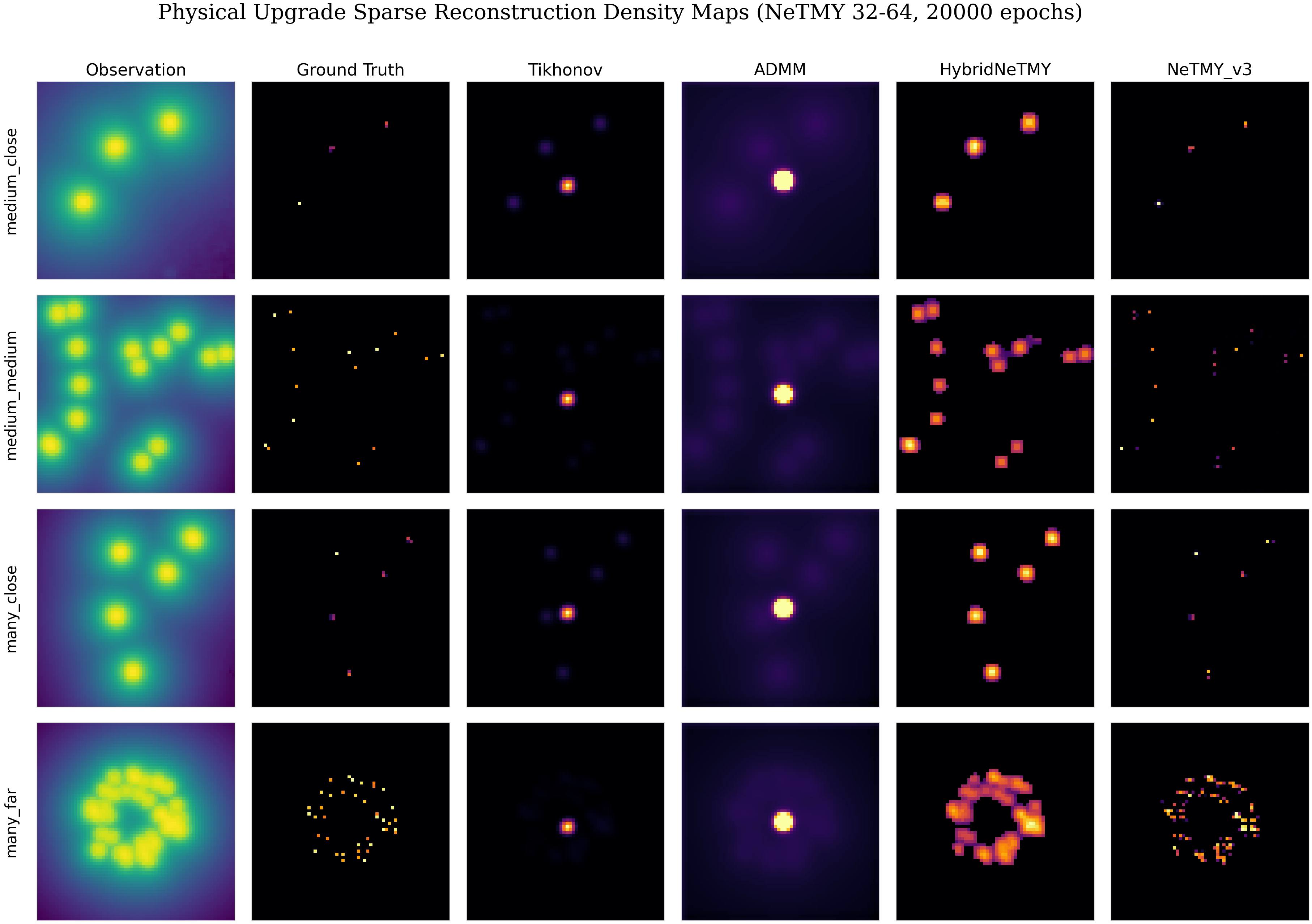}
    \caption{Matched-operator $\mathcal{F}_{2}/\mathcal{F}_{2}$ examples (samples $012$, $016$, $022$, $029$). Under the tensor power-summed operator the same ADMM run degrades to a centered collapse, while NeTMY produces the cleanest off-center reconstructions and tops every $\mathcal{F}_{2}/\mathcal{F}_{2}$ metric in Table~\ref{tab:matched}.}
    \label{fig:QTD-Physical}
\end{figure}



\end{document}